%% file: with-sq.tex
\documentclass{article}

\usepackage{amsmath}
\usepackage{amsthm}
\usepackage{amssymb}
\usepackage{times}
\usepackage{fullpage}

\theoremstyle{plain}\newtheorem{theorem}{Theorem}[section]
\theoremstyle{plain}\newtheorem{lemma}[theorem]{Lemma}
\theoremstyle{plain}\newtheorem{proposition}[theorem]{Proposition}
\theoremstyle{definition}\newtheorem{definition}[theorem]{Definition}
\theoremstyle{plain}\newtheorem{corollary}[theorem]{Corollary}

\newcommand{\R}{\ensuremath{\mathbb{R}}}
\newcommand{\bbN}{\ensuremath{\mathbb{N}}}
\newcommand{\bbR}{\ensuremath{\mathbb{R}}}
\newcommand{\mcC}{\ensuremath{\mathcal{C}}}
\newcommand{\mcJ}{\ensuremath{\mathcal{J}}}
\newcommand{\mcH}{\ensuremath{\mathcal{H}}}
\newcommand{\mcN}{\ensuremath{\mathcal{N}}}

\newcommand{\eps}{\ensuremath{\varepsilon}}
\newcommand{\sigmoid}{\ensuremath{\sigma_{\textrm{sig}}}}

\DeclareMathOperator*{\Var}{Var}
\DeclareMathOperator*{\Cov}{Cov}
\DeclareMathOperator*{\Prob}{P}
\DeclareMathOperator*{\Exp}{\mathbb{E}}

\DeclareMathOperator{\SDA}{SDA}
\DeclareMathOperator{\VSTAT}{VSTAT}
\DeclareMathOperator{\eSDA}{\eps-SDA}
\DeclareMathOperator*{\E}{\mathbb{E}}
\def\vol{{\sf vol}}
\DeclareMathOperator{\STAT}{STAT}
\DeclareMathOperator{\oneSTAT}{1-STAT}

\title{Gradient Descent for One-Hidden-Layer Neural Networks: Polynomial Convergence and SQ Lower Bounds\footnote{Accepted for presentation at the Conference on Learning Theory (COLT) 2019.}}

\author{Santosh Vempala \\
    Georgia Institute of Technology \\
    \texttt{vempala@gatech.edu}
\and
    John Wilmes  \\
    Brandeis University \\
    \texttt{wilmes@brandeis.edu}
}

\begin{document}

\maketitle

\begin{abstract}
We study the complexity of training neural network models with one hidden nonlinear activation layer and an output weighted sum layer.
We analyze Gradient Descent applied to learning a bounded target function on
    $n$ real-valued inputs. 
We give an agnostic learning guarantee for GD: starting from a
    randomly initialized network, it converges in mean squared loss to the minimum error (in $2$-norm) of the best approximation of the target function using a polynomial of degree at most $k$. Moreover, for any $k$, the size of the network and number of iterations needed are both bounded by $n^{O(k)}\log(1/\eps)$. The core of our analysis is the following existence theorem, which is of independent interest: 
for any $\eps > 0$, any bounded function that has a degree $k$ polynomial approximation with error
    $\eps_0$ (in $2$-norm), can be approximated to within error $\eps_0 + \eps$ as a linear
    combination of $n^{O(k)}\cdot \mbox{poly}(1/\eps)$ {\em randomly chosen} gates from any class
    of gates whose corresponding activation function has nonzero coefficients in its harmonic
    expansion for degrees up to $k$. In particular, this applies to training networks  of unbiased
    sigmoids and ReLUs.  We also rigorously explain the empirical finding that gradient descent
    discovers lower frequency Fourier components before higher frequency components. 
 
We complement this result with nearly matching lower bounds in the Statistical
    Query model. GD fits well in the SQ framework since each training
    step is determined by an expectation over the input distribution. We show
    that any SQ algorithm that achieves significant improvement over a constant
    function with queries of tolerance some inverse polynomial in the input
    dimensionality $n$ must use $n^{\Omega(k)}$ queries even when the target
    functions are restricted to a set of $n^{O(k)}$ degree-$k$ polynomials, and the input distribution is uniform
    over the unit sphere; for this class the information-theoretic lower bound is only $\Theta(k \log n)$. 

Our approach for both parts is based on spherical harmonics. We
    view gradient descent as an operator on the space of functions, and study
    its dynamics. An essential tool is the
    Funk-Hecke theorem, which explains the eigenfunctions of this operator in
    the case of the mean squared loss.
\end{abstract}

\input{intro.tex}

\input{prelim.tex}

\input{gd.tex}

\input{sq.tex}

\section{Discussion}

We have given a polynomial-time analysis of gradient descent for training a
neural network in an agnostic setting. In particular, we show that functions
that are \emph{approximated} by polynomials can be learned by gradient descent,
as well as functions computed by single-hidden-layer neural networks. These
results build on a long line of work by many authors studying the power of random
initialization combined with output-layer training.

We show that our analysis is essentially tight, in the sense that no
statistical query algorithm can have significantly better time complexity. 

Extending the training to hidden-layer weights cannot offer an asymptotic
improvement in the number of gates needed to achieve small error in the general
setting we consider. However, experiments suggest that training hidden-layer
weights might allow for tighter bounds in the realizable case. In particular,
it would be interesting to give a fully polynomial analysis of gradient descent
for learning data labeled by a single-hidden layer neural network with $m$
neurons. An extension for networks with bounded bias parameters, rather than
unbiased networks, would also be interesting.

\subsection*{Acknowlegments}

The authors are grateful to Adam Kalai and Le Song for helpful discussions. The authors also thank
Jo\"el Bella\"iche and the anonymous referees for careful reading and many suggestions that
improved the presentation.  This work was supported in part by NSF grants CCF-1563838, CCF-1717349
and E2CDA-1640081.

\bibliographystyle{plain}
\bibliography{SGD_bib}

\end{document}

%% file: intro.tex
\section{Introduction}

It is well known that artificial neural networks (NNs) can approximate any real-valued function.
Fundamental results \cite{hornik1989multilayer, cybenko1989approximation, barron1993universal} show
that a NN with a {\em single} hidden layer provides a universal representation up
to arbitrary approximation, with the number of hidden units needed depending on the function being
approximated and the desired accuracy.

Besides their generality, an important feature of NNs is the ease of training
them --- gradient descent (GD) is used to minimize the error of the network,
measured by a loss function of the current weights. This seems to work across a
range of labeled data sets. 
Yet despite its tremendous success, there is no satisfactory
explanation for the efficiency or effectiveness of this generic training algorithm\footnote{Indeed, one might consider this a miraculous feat of engineering and even ask, is there anything to explain rigorously? We are not entirely comfortable with this view and optimistic of some life beyond convexity.}.

The difficulty is that even for highly restricted classes of NNs,
natural loss functions such as the
mean squared loss have a highly non-convex landscape with many nonoptimal local
minima. However, when data is generated from a model with random weights, 
GD (the stochastic version with a small batch size)
seems to consistently learn a network with error close to zero. This raises the
prospect of a provable guarantee, but there are two complicating experimental
observations. First, the randomness of the initialization appears essential
(standard in practice) as in experiments it is possible to remain stuck at higher error.
Second, we observe smaller error (and it decreases more quickly) when the
model size used for training is made larger; in particular, for the realizable
case (when the data is itself labeled by a NN), we train using many more units than the original.
This aspect is also commonly encountered in the training of large NNs on real data --- even with
huge amounts of data, the size of the model used can be larger. 


In this paper we give nearly matching upper and lower bounds that help explain
the phenomena seen in practice when training NNs. The upper bounds
are for GD and the lower bounds are for all statistical query algorithms. We
summarize them here, and present them formally in the next section.

Our algorithmic result is an agnostic upper bound on the approximation error and time
and sample complexity of GD with the standard mean squared loss
function. Despite training only the output layer weights, our novel proof techniques
avoid using any convexity in the problem. Since our analysis does not
rely on reaching a global minimum, there is reason to hope the techniques will
extend to nonconvex settings where we can in general expect only to find a
local minimum. Prior results along this line were either for more complicated
algorithms or more restricted settings; the closest is the work of Andoni et
al.  \cite{apvz14} where they assume the target function is a bounded degree
polynomial. A detailed comparison of results is given in
Section~\ref{sec:related}. As a corollary of our convergence analaysis, we obtain a rigorous proof of the ``spectral bias'' of gradient descent observed experimentally in~\cite{spectral-bias}.

The upper bound shows that to get close to the best possible degree $k$
polynomial approximation of the data, it suffices to run GD on a
NN with $n^{O(k)}$ units, using the same number of samples. It suffices to train the
output layer weights alone. This is an agnostic guarantee. We prove a matching lower bound for solving
this polynomial learning problem over the uniform distribution on the unit sphere, for {\em any}
statistical query algorithm that uses tolerance inversely proportional to $n^{\Omega(k)}$.
Thus, for this general agnostic learning problem, GD is as good as it gets.




\subsection{Results}\label{sec:results}

We consider NNs with on inputs from the sphere $S^{n-1}\subseteq \bbR^n$, a single hidden layer
with $m$ units having some nonlinear activation $\phi: \bbR \to \bbR$, and a single linear output
unit. All units are without additive bias terms. We will consider inputs drawn from the uniform
distribution $D$ on $S^{n-1}$.

We denote by $W$ the set of units in the hidden layer, and abuse notation to write $u\in W$ for
both the hidden layer unit and its corresponding weight vector in $\bbR^n$. The output-layer weight
corresponding to $u \in W$ will be denoted $b_u \in \bbR$. Hence, the NN computes a function of the
form
\begin{equation}\label{eq:nn}
f(x) = \sum_{u \in W} b_u \phi(u \cdot x)\,.
\end{equation}
We initialize our NNs by choosing the vectors $u \in W$ independently from $D$, and setting each
$b_u$ to $0$.

For two functions $f,g:S^{n-1} \rightarrow \R$, the mean squared loss with respect to $D$ is
$\Exp_{x \sim S^{n-1}}((f(x)-g(x))^2)$.  Given data $(x,y)$ with $x \in \R^n$, $y\in \R$, we
analyze GD to minimize the mean squared loss of the current model with respect to the given data.
The specific GD procedure we consider is as follows: in each iteration, the gradient of the
loss function is computed using a finite sample of examples, with the entire sample reused for each
iteration. The output-layer weights $b_u$ are then modified by adding a fixed multiple of the
estimated gradient, and the hidden-layer weights $u \in W$ are kept fixed.

\paragraph{Convergence guarantees.}
Our first theorem is for training networks of sigmoid gates. The same statement holds for ReLU activation units and even functions $g$.

\begin{theorem}\label{thm:sigmoid}
    Let $\eps_0 > 0$, $k \in \bbN$, and
    $g : S^{n-1} \to \bbR$ an odd bounded function such that $\|g - g^{(\le
    k)}\|_2 \le \eps_0$, where $g^{(\le k)}$ denotes the best polynomial of degree at most $k$ 
    approximation of $g$ in $L^2$ norm on $S^{n-1}$. Then for any $\eps > 0$,  for some $m =
    n^{O(k)}\mbox{poly}(\|g\|_2/\eps)$ the following holds: 
    A randomly initialized single-hidden-layer NN with $m$ sigmoid gates in the hidden layer and a
    linear output layer, with high probability, will have mean squared loss of at most $\eps_0 +
    \eps$ after at most $n^{O(k)}\log(\|g\|_2/\eps)$ iterations of GD applied to the output layer
    weights, re-using a set of $m$ samples in each iteration.
\end{theorem}

Next we state a more general theorem. This will apply to a large class of
activation functions. The main property we need of the activation function is
that it should not be a low-degree polynomial. We first introduce additional notation. (See
Section~\ref{sec:harmonic} for related definitions and background.)
We denote by $\mcH_{n,k}$ the set of \emph{spherical harmonics} of degree $k$ on the sphere
$S^{n-1}$.
\begin{definition}
    Given an $L^2$ function $f$ on $S^{n-1}$, we denote by $f^{(k)}$ the
    projection of $f$ to $\mcH_{n,k}$, so $f = \sum_{k=0}^{\infty} f^{(k)}$. We also write
    $f^{(\le k)} = \sum_{i=0}^k f^{(i)}$, and for $S \subseteq \bbN$, we write
    $f^{(S)} = \sum_{i \in S} f^{(i)}$.
For $S \subset \bbN$ and $\alpha >0$, an $(n, S,\alpha)$-activation is a function
    $\phi: \R \to \R$ with the property that for any $u \in S^{n-1}$, the map $f(x) = \phi(u\cdot
x)$ has a harmonic polynomial expansion with $\|f^{(k)}\| \ge \alpha$ for all $k \in S$.
\end{definition}
Since the dimension $n$ is uniform throughout this paper, we will abbreviate our notation and refer
to $(S, \alpha)$-activations. The set $S$ will not generally depend on $n$, but the quantity
$\alpha = \alpha(n)$ generally will (see, e.g., Lemma~\ref{lem:taylor}).

For example, the commonly used sigmoid gate $\sigmoid(x) = 1/(1+e^{-x})$ is an
$(S,\alpha)$-activation function for $S$ the odd integers less than $k$
and $\alpha = n^{-O(k)}$. Similarly, ReLU gates are $(S,\alpha)$-activation
functions for subsets $S$ of the even integers.

\begin{theorem}\label{thm:general}
    Let $\eps_0 > 0$ and
    $g : S^{n-1} \to \bbR$ a bounded function such that $\|g - g^{(S)}\|_2
    \le \eps_0$.
    Then for any $\eps > 0$, and any $(S,\alpha)$-activation function $\phi$
    with $\|\phi\|_{\infty} \le 1$, for some $m = \mbox{poly}(1/\alpha, \|g\|_2/\eps)$
    the following holds: A randomly initialized single-hidden-layer NN with $m$ $\phi$-gates in the
    hidden layer and a linear output layer, with high probability, will have mean squared loss of
    at most $\eps_0 + \eps$ after at most $\mbox{poly}(1/\alpha)\log(\|g\|_2/\eps)$
    iterations of GD applied to the output layer weights, re-using a set of $m$ samples in each
    iteration.
\end{theorem}

This general theorem has the following corollary in the realizable case, when
data is generated by a one-hidden-layer NN. In this case, the function can be approximated by a
low-degree polynomial. In order to allow for this approximation guarantee, and to
side-step previous statistical query lower bounds~\cite{SVWX17}, we guarantee some degree on
nondegeneracy by focusing on unbiased NNs, i.e., networks without additive bias terms (as in
Eq.~\eqref{eq:nn}).

\begin{corollary}\label{cor:realizable}
    Let $g$ be computed by an unbiased one-hidden-layer NN with sigmoid units in the
    hidden layer and a linear output. Suppose the $\ell_1$ norm of the output
    layer weights is $a$, and each hidden layer weight vector has $\ell_2$ norm
    at most $b$. Then for every $\eps > 0$, for some $m = an^{O(b\log(ab/\eps))}$ the following
    holds: A randomly initialized
    single-hidden-layer NN with $m$ sigmoid units in the hidden layer and a
    linear output layer, with high probability, will have mean squared loss of
    at most $\eps$ after at most $n^{O(b\log(ab/\eps)}$ iterations of GD applied to the output
    layer weights, re-using a set of $m$ samples in each iteration.
\end{corollary}

The use of sigmoid units in Corollary~\ref{cor:realizable} is not essential,
but the bounds on network size and training time will depend on the specific
activation function chosen. 

%

\paragraph{Spectral bias.}

As a consequence of our techniques, we give a proof of the ``spectral bias'' phenomenon
observed experimentally in~\cite{spectral-bias}. The experiments of~\cite{spectral-bias} showed
that neural networks trained via gradient descent learned low Fourier frequencies more quickly than
higher frequencies, which the authors propose as a mechanism to explain generalization performance
of deep learning. We prove that low frequencies are indeed learned more quickly than high
frequencies, where ``low frequencies'' and ``high frequencies'' are understood as low and high
degree harmonic components of a function.

To quantify the relative speed of learning, we introduce additional notation.
\begin{definition}
    Let $H_i$ denote the residual after training a NN via GD for $i$ iterations, and let
    $\Delta_i^{(j)} = H_{i+1}^{(j)} - H_{i}^(j)$ denote the change in the $i$th residual in degree
    $j$.  Suppose $H_i^{(k)}, H_i^{(\ell)} \ne 0$. We say $r = r^{(k,\ell)}_i > 0$ is the
    \emph{rate of progress in degree $k$ relative to degree $\ell$} if
    \[
        \frac{\|\Delta_i^{(k)}\|}{\|\Delta_i^{(\ell)}\|} =
        r\frac{\|H_i^{(k)}\|}{\|H_i^{(\ell)}\|}\,.
    \]
\end{definition}
Thus, if the rate of progress in degree $k$ relative to degree $\ell$ is $r^{(k,\ell)}_i > 1$, the network
learns degree-$k$ information more quickly---and degree-$\ell$ information less quickly---compared
to what would be expected based on the relative sizes of the residual in degrees $k$ and $\ell$.
Conversely, if the rate $r^{(k,\ell)}_i < 1$, the network learns the degree-$k$ information more
slowly.

\begin{theorem}\label{thm:spectral-bias}
    Fix $\eps > 0$ and $\ell > k$.  For some $m = n^{O(\ell)}\mbox{poly}(1/\eps)$, the following
    holds for any odd bounded function $g : S^{n-1} \to \bbR$.
    A randomly initialized single-hidden-layer NN with $m$ sigmoid gates in the hidden layer and
    a linear output layer, trained via GD applied to the output layer weights, 
    re-using a set of $m$ samples in each iteration,
    with high probability will have rate of progress in degree $k$ relative to degree $\ell$ at
    least $r_i^{(k,\ell)} \ge n^{\Omega(\ell - k)}$, assuming the $i$'th residual $H_i$ satisfies 
    $\|H_i^{(\ell)}\|_2, \|H_i^{(k)}\|_2 \ge \eps$.
\end{theorem}

\paragraph{Lower bounds.}
Our lower bounds hold in the very general Statistical Query (SQ) model, first defined by Kearns \cite{Kearns93}.
An SQ algorithm solves a computational problem
over an input distribution and interacts with the input only by querying the expected value of 
of a bounded function up to a desired accuracy. For any integer $t >
0$ and distribution $D$ over $X$, a \textbf{$\VSTAT(t)$ oracle} \cite{FGRVX13} takes as input a \textbf{query function} $h: X \to
[0,1]$ with expectation $p=\E_D(h(x))$ and returns a value $v$ such that
\[
    \left|p - v\right| \le \max \left\{\frac{1}{t},
        \sqrt{\frac{p(1-p)}{t}}\right\}.
\]
The bound is the standard deviation of $t$ independent Bernoulli coins with desired expectation,
i.e., the error that even a random sample of size $t$ would yield.  The SQ complexity of an
algorithm is given by the number of queries and the batch size $t$. The remaining computation is
unrestricted and can use randomization. We will also give lower bounds against the $\oneSTAT$
oracle, which responds to queries with a single honest bit. Given a distribution $D$ over $X$ and a
query function $h : X \to \{0,1\}$, the $\oneSTAT$ oracle responds with a single value $h(x)$,
where $x \sim D$~\cite{yang2001learning}. 

The SQ framework was introduced by Kearns for supervised
learning problems~\cite{Kearns93} using the $\STAT(\tau)$ oracle, which,
for $\tau \in \bbR_{+}$,
responds to a query function $h: X \to [0,1]$ with a value $v$ such that
$|\E_D(h) - v| \le \tau$. The $\STAT(\sqrt{\tau})$ oracle can be simulated by the
$\VSTAT(O(1/\tau))$ oracle. 
The $\VSTAT$ oracle was introduced by
\cite{FGRVX13} who extended these oracles to more general problems
over distributions. 

Choosing a useful SQ model for regression problems is
nontrivial. We discuss some of the pitfalls in Section~\ref{sec:sq}. Our
lower bounds concern three query models.

The first allows quite general query functions. We say an SQ
algorithm (for regression) makes \emph{$L^{\infty}$-normalized
$\lambda$-Lipschitz queries} concerning an unknown concept $g:X \to \bbR$ if it
makes queries of the form $h: X \times [-1,1]
\to [0,1]$, where $h$ is $\lambda$-Lipschitz at any fixed $x \in X$, to which the SQ oracle should
respond with a value $v$ approximating $\Exp_{x \sim D}(h(x,
g(x)/\|g\|_{\infty}))$. 
We get similar lower bounds for a natural family of inner product queries with no Lipschitzness assumption.
We say an SQ algorithm makes \emph{inner product queries}
concerning an unknown concept $g:X \to \bbR$ if it makes queries of the form
$h: X\to [0,1]$ to an oracle that replies with an approximation of $\Exp_{x
\sim D}(g(x)h(x))$.
Finally, we say an SQ
algorithm makes \emph{$L^{\infty}$-normalized queries to $\oneSTAT$ with Gaussian noise of variance
$\eps$} concerning an unknown concept $g: X \to \bbR$ if it makes queries of the form $h: X \times
\bbR \to \{0,1\}$, to which the oracle replies with the value of $h(x, g(x)/\|g\|_{\infty}+\zeta)$ where $x \sim
D$ and $\zeta \sim \mcN(0,\eps)$. Unlike $\VSTAT$, the $\oneSTAT$ oracle, even with Gaussian noise, responds honestly and is not allowed to make any adversarial changes or coordinate its responses to multiple queries. 

\begin{theorem}\label{thm:sq-all}
    Let $\eps > 0$. For all $k, \lambda > 0$ and all
    sufficiently large $n$ and $d < \exp(n^{1/2-\eps})$, there exists
    a family $\mcC$ of degree-$k$ polynomials on $S^n$ with $|\mcC| = d$ such that
    if a randomized SQ algorithm learns $\mcC$ to regression error less than any fixed constant with
    probability at least $1/2$: 
    \begin{enumerate}
        \item[(1)] it requires at least $\Omega(d)$ queries, if the queries are
            inner product queries to $\VSTAT(n^{\Omega(k)})$;
        \item[(2)] it requires at least $\Omega(d)$ queries, if the queries are
    $L^{\infty}$-normalized $\lambda$-Lipschitz queries to
    $\VSTAT(n^{\Omega(k)}/\lambda)$;
\item[(3)] for $d = n^{\Omega(k)}$, it requires at least $n^{\Omega(k)}/\lambda$ queries, if the
    queries are $L^{\infty}$-normalized queries to $\oneSTAT$ with Gaussian noise of
            variance $1/\lambda^2$
    \end{enumerate}
where all the hidden constants depend on $\eps$ only.
\end{theorem}


In the case of training a NN via GD, the relevant
queries should yield gradients of the loss function with respect to the current
weights. In the case of mean squared loss, the gradients are of the form
$\Exp((g - f)\nabla_w f)$, where $g$ is the unknown concept, $f$ is the
current network, and $w$ represents parameters of $f$. These gradients can be
estimated via queries in any of the models we consider. The lower bound on the $\oneSTAT$ oracle in particular implies that
as long as the function values of queried inputs are perturbed by a random Gassian, any training algorithm needs $n^{\Omega(k)}$ queries.

\subsection{Approach and techniques}\label{sec:approach}
The gradient of the loss function with respect to any outer layer weight can be
viewed as a spherical transform of the current residual error. More precisely,
if the current function $f$ is computed by an unbiased single hidden-layer
NN with output-layer weights $b_u$, as in Eq.~\eqref{eq:nn}, and the residual error
with respect to the target function $g$ is $H=g-f$, then for any $u$,
\begin{equation}\label{eq:gradient-computation}
\nabla_{b_u} \|H\|^2 = 2\Exp_x( \phi(u\cdot x) H(x)).
\end{equation}
The latter expectation is quite special when the domain of integration is the
unit sphere. Different choices of the function $\phi$ correspond to
different spherical transformations. For example, $\phi(u\cdot x)$ being the
indicator of $u\cdot x \ge 0$ is the hemispherical transform, while $\phi(u\cdot x)
= 1$ iff $u\cdot x = 0$ is the Radon transform, etc. 
This type of transformation 
\[
        \mcJ_{\phi}(H)(u) = \Exp_{x \in S^{n-1}}(\phi(x \cdot u) H(x))
\]
has a closed form expression whenever the function $H$ is a harmonic polynomial (see definitions in
Section~\ref{sec:harmonic}). By the classical Funk-Hecke theorem, for any bounded function $\phi$
and any harmonic polynomial $P$ of degree $k$ on the sphere $S^{n-1} \subseteq \bbR^n$, there is an
explicit constant $\alpha_{n,k}(\phi)$ such that
\[
        \mcJ_{\phi}(P)(u) = \alpha_{n,k}(\phi) P(u)\,.
\]
In particular, the harmonic polynomials are eigenfunctions of the
operator $\mcJ_{\phi}$.  Moreover, since there exists an orthonormal basis of harmonic polynomials
for $L^2$ functions over the unit sphere, any function (in our case the residual $H$) has zero norm
iff the corresponding transform has zero norm (assuming the function $\phi$ has nonzero
coefficients $\alpha_{n,k}(\phi)$). 

With the above observations in hand, we can now outline our analysis.
We focus on the \emph{dynamics of GD as an operator
on a space of functions.} In particular, for a
set $Z \subseteq \bbR^n$ and function $f: \bbR^n \to \bbR$, we define an
operator
\begin{equation}\label{eq:empirical-grad}
    T_Z(f)(u) = \frac{1}{|Z|}\sum_{z \in Z}f(z)\phi(u\cdot z)\,.
\end{equation}
Thus, if the current residual error is given by some function $H$, then the
empirical gradient of the mean-squared loss with respect to a set $X$ of
labeled examples is $T_X(H)$ (see Section~\ref{sec:analysis}).

Our analysis proceeds in three stages:
\begin{enumerate}
    \item Show that, with a large enough set $X$ of samples, the empirical
        gradient operator $T_X$ approximates the Funk transform $\mcJ_\phi$ as
        an operator on the space of residual error functions
        (Lemmas~\ref{lem:fundamental-L2} and~\ref{lem:TXW-props})

    \item Bound the rate at which error from the approximation of $T_X$ by
        $\mcJ_{\phi}$ accumulates over multiple rounds of GD
        (Lemmas~\ref{lem:T2-approx} and~\ref{lem:infty-approx})

    \item Estimate the final loss in terms of the distance of the target
        function from the space of low-degree harmonic polynomials --- i.e.,
        the distance from the most significant eigenspaces of $\mcJ_\phi$
        (see proof of Lemma~\ref{lem:main-technical})
\end{enumerate}

A crucial observation that simplifies our analysis is that when $f$ is given by
a NN as in Eq.~\eqref{eq:nn}, then $f$ itself is obtained by
applying the operator $T_W$, where $W$ is the set of hidden weights in $f$, to
a function $a: S^{n-1} \to \bbR$ that computes the output-layer coefficients
for each gate (see Equation~\eqref{eq:f_i-def}).


Our analysis does not use the fact that the
optimization produces an approximate global minimum; hence, there is a greater
hope of generalizing to nonconvex regimes where we expect to instead only reach
a local minimum in general. Another pleasant feature of our analysis is that we
need not directly prove a ``representation theorem'' showing that the hypothesis minimizing the
population loss is a good approximation to the target function; instead, we can
derive such a result for free, as a corollary to our analysis. That is,
since we prove directly that GD on the output layer weights of a
single-layer NN with randomly-initialized gates results in small
loss, it follows that any low-degree harmonic polynomial is in fact
approximated by such a network. Our hope is that this new approach offers an interesting possibility for
understanding GD in more difficult settings.

A practical consequence of our method of analysis is that we can easily prove a ``spectral bias''
result, showing the lower degrees are learned more quickly than higher degrees, as was suggested
experimentally in~\cite{spectral-bias} (see Theorem~\ref{thm:spectral-bias}).


The upper bound guarantees hold for the agnostic learning problem of minimizing
the least squares error, and the bound is with respect to the best degree $k$
polynomial approximation. The size of the network needed grows as
$n^{\Omega(k)}$, as does the time and sample complexity. We show that this unavoidable for any SQ 
algorithm, including GD and its variants on arbitrary network
architectures. The ``hard" functions used for the lower bound will be generated
by spherical harmonic polynomials. Specifically, we use the univariate Legendre
polynomial of degree $k$ in dimension $n$, denoted as $P_{n,k}$, and also
called the Gegenbauer polynomial (see Section \ref{sec:harmonic} for more
background). We pick a set of unit vectors $u$ and for each one we get a
polynomial $f_u(x)=P_{n,k}(u\cdot x)$. We choose the vectors randomly so that
most have a small pairwise inner product. Then querying one of these
polynomials gives little information about the others (on the same input $x$),
and forces an algorithm to make many queries. As in the work on SQ
regression algorithms of~\cite{SVWX17}, it is essential not only to bound the
pairwise correlations of the ``hard'' functions themselves, but also of
arbitary ``smoothed'' indicator functions composed with the hard family. This
is accomplished by using a concentration of measure inequality on the
sphere to avoid regions where these indicators are in fact correlated. In
contrast to those earlier SQ regression lower bounds, we obtain bounds on the
sensitivity parameter $t$ for the $\VSTAT(t)$ oracle that scales with the
number of queries $d$ and the degree $k$.

\subsection{Related work}\label{sec:related}

Explaining the success of deep NNs and GD for
training NNs has been a challenge for several years. The trade-off
between depth and size for the purpose of representation has been rigorously
demonstrated \cite{telgarsky2016benefits, eldan2016power}. Moreover, there are strong
complexity-theoretic and cryptographic-assumption based lower bounds to contend
with \cite{BR92, Daniely16, Klivans16}. These lower bounds are typically based
on Boolean functions and ``hard" input distributions. More
recent lower bounds hold even for specific distributions and smooth functions,
for basic GD~\cite{Shamir16}, and even realizable smooth
functions for any SQ algorithm and any product logconcave input distribution
\cite{SVWX17}.  These earlier lower bound constructs are degenerate in the
sense that they rely on data generated by networks whose bias and weight vectors
have unbounded Euclidean norm as the dimension increases.  In contrast, the
constructions used in this paper match a corresponding upper bound almost
exactly by making use of generic harmonic polynomials in the construction,
apply to a significantly broader family of functions, and achieve a much
stronger bound on the sensitivity parameter $t$.

Upper bounds have been hard to come by. Standard loss functions, even for one-hidden-layer networks with an output sum gate, are not 
convex and have multiple disconnected local minima. One body of work shows how
to learn more restricted functions, e.g., polynomials \cite{apvz14} and
restricted convolutional networks \cite{bg17}. Another line of work
investigates classes of such networks that can be learned in polynomial time,
notably using tensor methods \cite{Janzamin15, Sedghi16} and polynomial kernels
\cite{GKKT16, Goel2017}, more direct methods with assumptions on the structure of the network \cite{goel2018learning, ge2018learning}   and a combination of tensor initialization followed by
GD~\cite{zhong17}. A recent paper shows that the tensor method
can be emulated by GD by adding a sufficiently sophisticated
penalty to the objective function \cite{GLM17}. Earlier work gave combinatorial
methods to learn random networks \cite{arora14}, guarantees for learning linear
dynamical systems by GD~\cite{HardtMR16} and ReLU networks with
more restrictive assumptions \cite{LiYuan17}. Representation theorems analogous
to our own were also proved in~\cite{bach17}, and a very general analysis of
GD is given in~\cite{Daniely17}.

Our analysis is reminiscent of the well-known random kitchen sinks paper
\cite{Rahimi-Recht08}, which showed that GD using a hard upper
bound on the magnitude of coefficients (in practice, an $L_1$ penalty term)
with many random features from some distribution achieves error that converges
to the best possible error among functions whose coefficients are not much
higher than those of the corresponding densities of the sampling distribution.
While this approach has been quite insightful (and effective in practice), it
(a) does not give a bound for standard GD (with no penalty) and
(b) does not address functions that have very different support than the
sampling distribution. Our bounds compare with the best possible polynomial
approximations and are essentially the best possible in that generality for
randomly chosen features. 

The work of Andoni et al.~\cite{apvz14} shows that GD applied to learn
a bounded degree polynomial, using a 1-hidden-layer network of exponential
gates, converges with roughly the same number of gates (and a higher iteration
count, $\mbox{poly}(1/\eps)$ instead of $\log(1/\eps)$ to achieve error $\eps$). A
crucial difference is that our analysis is agnostic and we show that 
GD converges to the error of the {\em best} degree $k$ approximation of
the target function given sufficient many gates. We also state our results for
general and commonly-used activation functions, rather than the $e^z$ gate
analyzed in~\cite{apvz14}, and obtain explicit sample complexity bounds. Of
course, the proof technique is also novel; we obtain our representation theorem
as a side effect of our direct analysis of GD, rather than the other way around.


%% file: prelim.tex
\section{Spherical Harmonics}\label{sec:harmonic}


We now recall the basic theorems of spherical harmonics we will require. A more detailed treatment
can be found in~\cite{groemer-sh}.

A homogeneous polynomial $p$ of degree $k$ in $\bbR^n$ is said to be
\emph{harmonic} if it satisfies the differential equation $\Delta p = 0$, where
$\Delta$ is the Laplacian operator. We denote by $\mcH_{n,k}$ the set of
\emph{spherical harmonics} of degree $k$ on the sphere $S^{n-1}$, i.e., the
projections of all harmonic polynomials of degree $k$ to the sphere $S^{n-1}$.
The only properties of harmonic polynomials used in this paper are that they
are polynomials, form an orthogonal basis for $L^2(S^{n-1})$, and are
eigenfunctions of Funk transforms, as we now explain.
We denote by $P_{n,k}: \bbR \to \bbR$ the (single-variable) Legendre
polynomial of degree $k$ in dimension $n$, which is also called the Gegenbauer
polynomial. We note that $|P_{n,k}(t)| \le 1$ for all $|t| \le 1$.

We define $N(n,k) = |\mcH_{n,k}|$.

\begin{definition}\label{def:transform}
    Let $\phi: \bbR \to \bbR$ be bounded and integrable. We define the Funk
    transformation for functions $H:S^{n-1}\to\bbR$ as
        $\mcJ_{\phi}(H)(u) = \Exp_{x \in S^{n-1}}(\phi(x \cdot u) H(x))$.
For $n, d \in \bbN$ we define the constant
    \[
        \alpha_{n,k}(\phi) = \int_{-1}^1 \phi(t) P_{n,k}(t)(1-t^2)^{(n-3)/2}dt\,.
    \]
\end{definition}

\begin{theorem}[Funk--Hecke]\label{thm:FH}
    Let $\phi : [-1,1] \to \bbR$ be bounded and integrable, and let $P \in
    \mcH_{n,k}$. Then, for $\mcJ_{\phi}$ and $\alpha_{n,k}(\phi)$ as in
    Definition~\ref{def:transform},
        $\mcJ_{\phi}(P)(u) = \alpha_{n,k}(\phi) P(u)$.
\end{theorem}

The following proposition is immediate from Cauchy-Schwarz.
\begin{proposition}\label{prop:T-Linf}
    We have $\|\mcJ_{\phi}(G)\|_{\infty} \le \|\phi\|_2\|G\|_2$.
\end{proposition}

\begin{lemma}\label{lem:T-expansion}
Let $\phi : [-1,1] \to \bbR$ be bounded and integrable, and let $H: \bbR^n \to
    \bbR$. Then for any $k \in \bbN$, $(\mcJ_{\phi}H)^{(k)} =
    \alpha_{n,k}(\phi) H^{(k)}$.
\end{lemma}
\begin{proof}
    By Proposition~\ref{prop:T-Linf}, $\mcJ_{\phi}$ has bounded norm 
    as an operator on $L^2(S^{n-1})$ and so by Theorem~\ref{thm:FH},
    \[
        \mcJ_{\phi}(H) =
        \mcJ_{\phi}\left(\sum_{k=0}^{\infty}H^{(k)}\right)
        = \sum_{k=0}^{\infty} \mcJ_{\phi}(H^{(k)})
        = \sum_{k=0}^{\infty} \alpha_{n,k}(\phi)H^{(k)}\,. \qedhere
    \]
\end{proof}

\begin{lemma}\label{lem:S-alpha-estimate}
    Let $\phi : \bbR \to \bbR$ be an $(S,\alpha)$-activation. Then for any $f : S^{n-1} \to \bbR$,
    we have
    \[
        \|f - \mcJ_{\phi}^2f\|_2^2 \le \|f\|_2^2 - \alpha^2\|f^{(S)}\|_2^2\,.
    \]
\end{lemma}
\begin{proof}
    By Lemma~\ref{lem:T-expansion},
    \begin{align*}
        \|f - \mcJ_{\phi}^2f\|_2^2
        &= \sum_{k=0}^{\infty}(1-\alpha_{n,k}(\phi)^4)\|f^{(k)}\|_2^2 \\
        &= \|f\|_2^2 - \sum_{k\in S}\alpha_{n,k}(\phi)^4\|f^{(k)}\|_2^2 \\
        &\le \|f\|_2^2 - \alpha^4\|f^{(S)}\|_2^2\,.
    \end{align*}
\end{proof}


\subsection{Spectra for specific activation functions}

We first prove a general lemma describing the harmonic spectrum of a wide class
of functions, and then derive estimates of the spectra for commonly used activation
functions.

\begin{lemma}\label{lem:taylor}
    Suppose $\phi : [-1,1] \to \bbR$ has an absolutely convergent Taylor series
    $\phi(t) = \sum_{i=0}^{\infty}a_it^i$ on $[-1,1]$. Suppose that for all
    $i > j$, we have $a_i < a_j$ whenever $a_i$ and $a_j$ are nonzero. Then for
    any positive integer $d$, $\phi$ is an
    $(S,a_dn^{-d-O(1)})$-activation, where $S = \{ i \le d : a_i \ne 0\}$.
\end{lemma}
\begin{proof}
    Define
    \[
        r_{n,k} = \frac{(-1)^k\Gamma((n-1)/2)}{2^k\Gamma(k+(n-1)/2)}\,
    \]
By Rodrigues' formula (see~\cite[Proposition 3.3.7]{groemer-sh}),
    \[
        \alpha_{n,k}(\phi) = \int_{-1}^1\phi(t)P_{n,k}(t)(1-t^2)^{(n-3)/2}dt
        =
        r_{n,k}\int_{-1}^1\phi(t)\frac{d^k}{dt^k}(1-t^2)^{k+(n-3)/2}\,.
    \]
    Hence, by the bounded convergence theorem,
    \[
        \alpha_{n,k}(\phi) = 
        r_{n,k}\sum_{i=0}^{\infty}a_i\int_{-1}^1t^i\frac{d^k}{dt^k}(1-t^2)^{k+(n-3)/2}\,.
    \]

    We claim that
    \begin{equation}
        \int_{-1}^1t^i\frac{d^k}{dt^k}(1-t^2)^{k+(n-3)/2} = 
            \left\{\begin{matrix} 0 & i < k \textrm{ or } i \not\equiv k
                \bmod 2 \\
                (-1)^kk!B((i-k+1)/2, (n-3)/2+k+1) & \textrm{otherwise}
        \end{matrix}\right.
    \end{equation}
    where $B(a,b)$ is the Euler beta function.
    Indeed, integrating by parts, we see that if $i < k$ the expression is $0$,
    and otherwise
    \[
        \int_{-1}^1t^i\frac{d^k}{dt^k}(1-t^2)^{k+(n-3)/2} =
        (-1)^kk!\int_{-1}^1t^{i-k}(1-t^2)^{k + (n-3)/2}\,.
    \]
    After a change of variables $u=t^2$, this latter integral is by definition
    $B((i-k+1)/2, (n-3)/2+k+1)$.

    Therefore, we compute for all $k \ge 0$,
    \[
        \alpha_{n,k}(\phi) = r_{n,k}(-1)^kk!\sum_{i=l}^{\infty}a_iB((i-k+1)/2, (n-3)/2+k+1)
    \]
    Now for any $i > j \ge k$ of the same parity $\bmod 2$, if $a_j \ne 0$ we estimate
    \begin{align*}
        \left|\frac{a_iB((i-k+1)/2, (n-3)/2 +k+1)}{a_jB((j-k+1)/2,
        (n-3)/2 +k+1)}\right| < 1/2\,.
    \end{align*}
    In particular, whenever $a_k \ne 0$ we have
    \begin{align*}
        |\alpha_{n,k}(\phi)| = \Theta(r_{n,k}k!a_k B(1/2, (n-3)/2+k+1))
        &= \Theta(a_kn^{-k-O(1)})\,.
    \end{align*}
\end{proof}

\begin{lemma}\label{lem:sigmoid-spectrum}
    \begin{enumerate}
        \item Let $\phi(t) = 1/(1+e^{-t})$ be the standard sigmoid function. Then for any positive integer $d$, $\phi$ is
    an $(S, n^{-d-O(1)})$-activation function, where $S$ contains $0$ and all
    odd integers less than $d$.
\item Let $\phi(t) = \log(e^t + 1)$ be the ``softplus'' function. Then for any positive integer $d$, $\phi$ is
    an $(S, n^{-d-O(1)})$-activation function, where $S$ contains $1$ and all
    even integers less than $d$.
    \end{enumerate}
\end{lemma}
\begin{proof}
    The statement follows from Lemma~\ref{lem:taylor} by computing the relevant
    Taylor series.
\end{proof}

We can also perform a similar computation for ReLU activations. (A more general
estimate is given in~\cite[Appendix D.2]{bach17}.)

\begin{lemma}\label{lem:relu-spectrum}
    Let $\phi(t) = \max\{t, 0\}$ be the ReLU function. Then for any positive
    integer $d$, $\phi$ is an $(S,n^{-d-O(1)})$-activation function, where $S$
    contains $1$ and all even integers less than $d$.
\end{lemma}


%% file: gd.tex
\section{Analysis of Gradient Descent}\label{sec:analysis}

In this section, we fix a function $g: \bbR^n \to \bbR$ we wish to learn.
We also fix an $(S,\alpha)$-activation function
$\phi$ with $\|\phi\|_{\infty} \le 1$, for some finite $S \subseteq \bbN$.

We let $W \subseteq \bbR^n$ be a finite set of independent points drawn from the uniform
distribution $D$ on the sphere $S^{n-1}$. Similar to Eq.~\eqref{eq:nn}, we define $f : \bbR^W
\times \bbR^n \to \bbR$ by $f(b,x) = f_b(x) = \sum_{u \in W} b_u \phi(u \cdot x)$, so $f$ is
computed by an unbiased single-hidden-layer neural network with hidden layer weight matrix given by
$W$ and linear output layer weights given by $b$.  We will study how $f$ changes as we update $b$
according to gradient descent on the mean-squared loss function $\Exp_{x \sim D}(f(b,x) - g(x))^2$.
We will state bounds in terms of some of the parameters, and then show that for adequate choices of
these parameter, gradient descent will succeed in reducing the loss below an arbitrary threshold,
proving Theorem~\ref{thm:general}.

We now define notation that will be used throughout the rest of this section.

We fix $\eps > 0$, the approximation error we will achieve over the projection
of $g$ to harmonics of degrees in $S$.
We define quantities $t$, $\delta$, and $m$ as follows, using absolute
constants $c_t$, $c_\delta$, and $c_m$ to be defined later in the proof. The maximum number of
iterations of gradient descent will be
\begin{equation}\label{eq:c_t}
    t = c_t\alpha^{-4}\log\left(\frac{\|g\|_2}{\eps}\right)\,.
\end{equation}
We define $\delta$ to be an error tolerance used in certain estimates in the
proof,
\begin{equation}\label{eq:c_delta}
    \delta = c_{\delta}\alpha^4\left(\frac{\eps}{\|g\|_2t}\right)^2\,.
\end{equation}
Finally, we define $m$ to be the number of hidden units (so $|W| = m$), as well as the number
of samples,
\begin{equation}\label{eq:c_m}
    m = c_m\frac{\|g\|_{\infty}}{\delta^2}\log\left(\frac{\|g\|_{\infty}}{\delta}\right)\,.
\end{equation}

Let $X$ be a collection of $m$ random independent samples $x \in \bbR^n$,
The set $X$, along with the labels $g(x)$ for $x \in X$, will be the training
data used by the algorithm.\footnote{We remark that there is no need to
have $|X| = |W|$; our analysis simply achieves the same bound for both. It is,
however, useful to have separate sets $X$ and $W$, so that these samples are
independent.}


We recall the definition in Eq.~\eqref{eq:empirical-grad} of the operator
\[
    T_Z(f)(u) = \frac{1}{|Z|}\sum_{z \in Z}f(z)\phi(u\cdot z)\,.
\]
defined for sets $Z \subseteq \bbR^n$ and functions $f: \bbR^n \to \bbR$.
As described in Section~\ref{sec:approach}, the empirical gradient is given by the operator
$T_X$ applied to the residual error, i.e., the gradient of $(g- f)^2$ with
respect to the output layer weight for the gate $u$ is estimated as $T_X(g-f)(u)$, where $X$ is the
set of sample inputs. On the other hand,
we will observe below that the neural network $f$ itself can also be understood as
the result of applying the operator $mT_W$ to a function representing the
output-layer weights.

For integers $i \ge 0$ we shall define functions $f_i, a_i : \bbR^n \to \bbR$ recursively,
corresponding to the model function $f$ and its coefficients after $i$ rounds
of gradient descent.  In particular, we
let $f_i(x) = f(a_i, x)$, i.e.,
\begin{equation}\label{eq:f_i-def}
    f_i(x) = \sum_{u \in W} a_i(u) \phi(u \cdot x) = mT_W(a_i)(x)
\end{equation}
We denote by $H_i = g - f_i$ the $i$th residual.
We define $a_0(u) = 0$ and, for $i \ge 1$, set
    $a_i(u) = a_{i-1}(u) + (1/m)T_XH_i(u)$.

We therefore have the following two propositions which describe how the
neural network evolves over multiple iterations of gradient descent. 

\begin{proposition}\label{prop:gd-identity}
    Suppose the output-layer weights $b \in \bbR^W$ are initially $0$. Then after $i$ rounds of
    gradient descent with learning rate $1/(2m)$, we have $b_u = a_i(u)$.
\end{proposition}
\begin{proof}
    Indeed, as we have observed in Eq.~\eqref{eq:gradient-computation}, for
    each $u \in W$, the true gradient of the loss $(g-f)^2$ with respect to the
    output-level weight $b_u$ is $2\Exp_x( \phi(u\cdot x) (g-f)(x))$. So the
    empirical gradient using the samples in $X$ is indeed 
    \[
        \frac{2}{|X|}\sum_{x \in X} \phi(u\cdot x) (g-f)(x) = 2T_X(g-f)(u)\,.
    \]
    Thus, a single iteration of gradient descent with learning rate $1/(2m)$
    will update the weight $b_u$ by adding $(1/m)T_X(g -f)(u)$. The proposition
    now follows by induction on $i$.
\end{proof}

\begin{proposition}\label{prop:f-from-T}
    For all $i \ge 0$, $f_{i+1} = f_i + T_WT_XH_i$.
\end{proposition}
\begin{proof}
    By the definitions of $f_i$ and $a_i$, we have
    \begin{align*}
        f_{i+1}(x) &= \sum_{u \in W} a_{i+1}(u) \phi(u \cdot x) \\
                   &= \sum_{u \in W} (a_i(u) + (1/m)T_XH_i(u))\phi(u \cdot x) \\
                   &= f_i(x) + T_WT_XH_i(x)
    \end{align*}
    as desired.
\end{proof}

Having introduced and explained the necessary notation, we now state our main technical
estimate, the following Lemma~\ref{lem:main-technical}, which will be proved at the end of this section. For the rest of Section~\ref{sec:analysis},
we write $\Delta H_i = H_{i+1} - H_i = f_{i+1} - f_i$ for the change in the residual at step $i$,
and we abbreviate $\mcJ = \mcJ_{\phi}$.

\begin{lemma}\label{lem:main-technical}
    Suppose $i \le t$ and $\|H_j^{(S)}\|_2 \ge \eps$ for all $j \le i$. Then with high probability
\[
    \|\Delta H_i - \mcJ^2H_i\|_2 = O(\delta\|g\|_2t^2)\,.
\]
\end{lemma}

\subsection{Proof of Main Results}

Given Lemma~\ref{lem:main-technical}, proved in the following Section~\ref{sec:main-lemma}, the
main results stated in Section~\ref{sec:results} are straightforward.

\begin{proof}[Proof of Theorem~\ref{thm:general}]
    By Lemma~\ref{lem:main-technical}, as long as $\|H_i\|_2^2$ remains larger than $\eps$ and $i
    \le t$, we have $\|\Delta H_i - \mcJ^2H_i\|_2 \le O(\delta\|g\|_2t^2)$. Now
    $(\Delta H_i - \mcJ^2H_i)^{(S)}$ and 
    $(\Delta H_i - \mcJ^2H_i)^{(\overline{S})}$ are orthogonal, so also $\|(\Delta H_i - \mcJ^2
    H_i)^{(S)}\|_2 \le O(\delta\|g\|_2t^2)$. Therefore, rewriting $\Delta H_i = H_{i+1} - H_i$,
    we have
    \[
        \|H_{i+1}^{(S)}\|_2 \le \|H_i^{(S)}\|_2 + O(\delta \|g\|_2t^2)\,.
    \]
    Combining with Lemma~\ref{lem:S-alpha-estimate},
    \[
        \|H_{i+1}^{(S)}\|_2^2 \le (1 - \alpha^4)\|H_i^{(S)}\|_2^2 + O(\delta\|g\|_2^2t^2)\,.
    \]
    For a sufficiently small choice of the constant $c_{\delta}$ defining $\delta$
    (Eq.~\eqref{eq:c_delta}), under the assumption that $\|H_i^{(S)}\|_2^2 \ge \eps$, we can take
    the $O(\delta\|g\|_2^2t^2)$ term to be at most $(\alpha^4/2)\|H_i^{(S)}\|_2^2$.
    Therefore,
    \[
        \|H_{i+1}^{(S)}\|_2^2 \le \left(1- \frac{\alpha^4}{2}\right)\|H_i^{(S)}\|_2^2\,.
    \]
    Since $\|H_0^{(S)}\|_2 \le \|H_0\|_2 = \|g\|_2$, for some $s = O(\alpha^{-4}\log(\|g\|/\eps)) < t$ we have
    $\|H_s^{(S)}\|_2^2 < \eps$ (assuming a sufficiently large choice of the constant $c_t$ defining
    $t$ in Eq.~\eqref{eq:c_t}). Then
    \[
        \|H_s\|_2^2 = \|H_s^{(\overline{S})}\|_2^2 + \|H_s^{(S)}\|_2^2 < \|H_s^{(\overline{S})}\|_2^2 + \eps
    \]
    as desired.
\end{proof}

Theorem~\ref{thm:sigmoid} now follows from Theorem~\ref{thm:general}, in view
of Lemma~\ref{lem:sigmoid-spectrum}. \qedhere

To prove Corollary~\ref{cor:realizable}, we first recall an approximation lemma of Livni et
al.~\cite[Lemma 2]{lss-train14}:

\begin{lemma}\label{lem:sigmoid-approx}
    Let $\phi(t) = 1/(1+e^{-t})$ denote the sigmoid function.
    For every $\eps > 0$, there is a polynomial $p$ of degree $d =
    O(L\log(L/\eps))$ such that $|p(t) - \phi(t)| < \eps$ for all $t \in
    [-L,L]$.
\end{lemma}

\begin{proof}[Proof of Corollary~\ref{cor:realizable}]
    Set $\eps' = \eps/a$. By Lemma~\ref{lem:sigmoid-approx}, there is a
    polynomial $p$ of degree $O(b\log(b/\eps'))$ such that $|p(t) - \phi(t)|
    < \eps$ for all $t \in [-b,b]$. Therefore, for every $u \in \bbR^n$ with
    $\|u\|_2 \le b$ and every $x \in S^{n-1}$, we have $|p(u\cdot x) -
    \phi(u\cdot x)| \le \eps'$. Hence, $|\sum_i a_i (p(u_i\cdot x) -
    \phi(u_i\cdot x)| < \eps$ whenever $\sum_i |a_i| < a$ and each $u_i$
    satisfies $\|u_i\|_2 \le b$. In particular, the functions computed by the
    networks described in the statement of the corollary can be approximated to
    within $\eps$ error by polynomials of degree $O(b\log(ab/\eps))$. The
    Corollary now follows from Theorem~\ref{thm:sigmoid}.
\end{proof}

Finally, we prove Theorem~\ref{thm:spectral-bias}.

\begin{proof}[Proof of Theorem~\ref{thm:spectral-bias}]
    In Eqs.~\eqref{eq:c_delta} and~\eqref{eq:c_m} defining $\delta$ and $m$, we defined $m$ in
    terms of $\delta$, but what matters for the proof of Lemma~\ref{lem:main-technical} is that $m
    = \Omega(\delta^{-2})$, and if the number of units $m$ is increased beyond the bound necessary
    for Theorem~\ref{thm:general}, we may decrease $\delta$ to preserve the relationship between
    $m$ and $\delta$ (and similarly $t$). So without loss of generality, we may take
    $\delta\|g\|_2t^2 > m^{-\Omega(1)}$. Therefore, taking $m =
    n^{O(\ell)}\mbox{poly}(1/\eps)$, for $j \in \{k, \ell\}$ 
    \[
        \alpha_{n,j}(\phi)^2\|H_i^{(j)}\|_2 < n^{-\Omega(j)}\eps < \delta\|g\|_2t^2\,.
    \]
    Using Lemma~\ref{lem:main-technical}, we compute
    \begin{align*}
        r_i^{(k,\ell)} = \frac{\|\Delta_i^{(k)}\|\|H_i^{(\ell)}\|}{\|\Delta_i^{(\ell)}\|\|H_i^{(k)}\|}
        &\ge \frac{(\|\mcJ^2 H_i^{(k)}\| - O(\delta\|g\|_2t^2))\|H_i^{(\ell)}\|}{(\|\mcJ^2
        H_i^{(\ell)}\| + O(\delta\|g\|_2t^2))\|H_i^{(k)}\|} \\
        &= \frac{(\alpha_{n,k}(\phi)^2 \|H_i^{(k)}\| - O(\delta\|g\|_2t^2))\|H_i^{(\ell)}\|}{(
        \alpha_{n,\ell}(\phi)^2\|H_i^{(\ell)}\| + O(\delta\|g\|_2t^2))\|H_i^{(k)}\|} \\
        &= \Omega\left(\frac{\alpha_{n,k}(\phi)^2}{\alpha_{n,\ell}(\phi)^2}\right) \\
        &= n^{\Omega(\ell - k)}\,. \qedhere
    \end{align*}
\end{proof}

\subsection{Proof of Lemma~\ref{lem:main-technical}}\label{sec:main-lemma}

We now prove Lemma~\ref{lem:main-technical}.
Essentially, the lemma states that 
the operator $T_Z$ approximates $\mcJ$ for sufficiently large sets $Z$. 
We will prove Lemma~\ref{lem:main-technical} via a sequence of gradually improving estimates of the
approximation of $\mcJ$ by $T_Z$.
Lemma~\ref{lem:fundamental-L2} gives a very
general approximation, which we use to prove the finer
approximation described in Lemma~\ref{lem:TXW-props}.

\begin{lemma}\label{lem:fundamental-L2}
    Let $f:\bbR^n \to \bbR$ and $u \in \bbR^n$, and let $\eta, p > 0$.
    There is some
    \[
        \ell = O\left(\frac{\|f\|_2^2 +
        \eta\|f\|_{\infty}}{\eta^2}\log\left(\frac{1+\|f\|_{\infty}}{\eta p}\right)\right)
    \]
    such that if $Z \subseteq \bbR^n$ is a set
    of $\ell$ independent random points drawn from $D$, 
    then with probability at least $1 - p$, we have both
        $\|T_Z(f) - \mcJ(f)\|_2 \le \eta$
    and
        \[\Prob_{u \sim D}\left(|T_Z(f)(u) - \mcJ(f)(u)| >
        \eta/2\right) < p\,.\]
\end{lemma}
\begin{proof}
    Without loss of generality assume $\eta < 1$ and 
    let $p_0 = \eta^2p^2/(4(1 + 2\|f\|_{\infty})^2)$.
    Fix $u \in \bbR^n$. Let $z_1,\ldots, z_{\ell}$ be drawn independently from $D$, and
    let $\zeta_i$ denote the random variable $f(z_i)\phi(u\cdot z_i) - \mcJ(f)(u)$ for $1 \le i \le \ell$.
    We have $\Exp(\zeta_i) = 0$ by definition of $\mcJ$. Since $\|\phi\|_{\infty} \le 1$, we have by H\"older's
    inequality
    \[
        \Var(\zeta_i) \le \|\phi\|_{\infty}^2\|f\|_2^2 \le
        \|f\|_2^2\,.
    \]
    Furthermore, by Proposition~\ref{prop:T-Linf},
    \[
        |\zeta_i| \le \|\phi_{\infty}\|f\|_{\infty} + \|\phi\|_2\|f\|_2 \le 2\|f\|_{\infty}\,.
    \]
   By a Bernstein bound, we therefore have
    \begin{align*}
        \Prob_Z\left(
        \left|\frac{1}{\ell}\sum_{i=1}^{\ell}\zeta_i\right| > \eta/2\right) &<
        2\exp\left(-\Omega\left(\frac{\ell\eta^2}{\|f\|_2^2 +
        \eta\|f\|_{\infty}}\right)\right) < p_0\,,
    \end{align*}
    with the second inequality holding for an appropriate choice of the constant hidden in the
    definition of $\ell$.

    Let $Z = \{z_1,\ldots, z_{\ell}\}$, so 
    \[
        \frac{1}{\ell}\sum_{i=1}^{\ell}\zeta_i = T_Z(f)(u) - \mcJ(f)(u)\,.
    \]
    Let $B(u,Z) = 1$ if $|T_Z(f)(u) - \mcJ(f)(u)| > \eta/2$ and $0$ otherwise.
    By the preceding inequality, we have $\Exp_{u,Z}(B(u,Z)) < p_0$.
    Therefore, by Markov's inequality, the probability over the choice of $Z$
    that $\Exp_u(B(u,Z)) > p_0/p$ is at most $p$. Hence, with probability $1-p$
    over the choice of $Z$, we have
    \begin{equation}\label{eq:fL2-good}
        \Prob_u(|T_Z(f)(u) - \mcJ(f)(u)| > \eta/2) = \Exp_u(B(u,Z)) <
        \frac{\eta^2p}{4(1+ 2\|f\|_{\infty})^2}\,.
    \end{equation}
    In particular, the second inequality of the present lemma holds.
    
    We now complete the proof of the first inequality.  For all choices of $Z$ and $u$, using
    Proposition~\ref{prop:T-Linf}, we have
    \[
       |T_Z(f)(u) - \mcJ(f)(u)| < |T_Z(f)(u)| + |\mcJ(f)(u)| \le
        2\|f\|_{\infty}\,.
    \]
    Combined with Eq.~\eqref{eq:fL2-good}, we have
    \[
        \|T_Z(f) - \mcJ(f)\|_2^2 \le (2\|f\|_{\infty})^2\Prob(|T_Z(f)(u) - \mcJ(f)(u)| >
        \eta/2) + \eta^2/4 \le \eta^2\,.\qedhere
    \]
\end{proof}

We denote by $\phi_x : \bbR^n \to \bbR$ the function $\phi_x(u) =
\phi(u\cdot x)$.

In the following Lemma~\ref{lem:TXW-props} we prove a finer-tuned
approximation of the operator $\mcJ$ by both $T_X$ and $T_W$. Since
Lemma~\ref{lem:fundamental-L2} doesn't give a sufficiently tight approximation
between the operators simultaneously for every $L^2$-function on $S^{n-1}$, we
restrict our attention to the subspace we care about, namely, the functions
spanned by the $\phi_x$ for $x \in W\cup X$.

\begin{lemma}\label{lem:TXW-props}
    With probability $1-1/m$ over the choice of $W$ and $X$, the following 
    statements are all true:
    \begin{enumerate}
        \item[(1)] $\|T_Xg - \mcJ g\|_2 \le \delta$;
        \item[(2)] For all $u \in W$, we have $\|T_X\phi_u - \mcJ
            \phi_u\|_2 \le \delta$;
        \item[(3)] For all $x \in X$ we have $\|T_W\phi_x - \mcJ
            \phi_x\|_2 \le \delta$
        \item[(4)] For all $x\ne y \in X$, we have $|T_W(\phi_x)(y) -
            \mcJ(\phi_x)(y)| \le \delta/2$;
    \end{enumerate}
\end{lemma}
\begin{proof}
We will use Markov's inequality to bound the probability that
$T_W(\phi_x)$ is far from $\mcJ(\phi_x)$ at a random input, followed by a union bound over the
choice of $W$ and the choice of $X$. We require the constant $c_m$ to be sufficiently large.

    In detail, we have
        $m \ge c_m\log(\|g\|_{\infty}/\delta)(\|g\|_{\infty}/\delta^2)$, so
for any fixed $k, c_0 > 0$, we can set $c_m$ sufficiently large that there is some
    $p < 1/m^k$ also satisfying
    \[
        m \ge c_0\frac{\|g\|_{\infty}}{\delta^2}\log\left(\frac{\|g\|_{\infty}}{\delta p}\right)\,.
    \]
The same statement also holds (for appropriate choice of $c_m$) with $\phi$ in place
    of $g$, since $\|\phi\|_{\infty} \le 1$.
    Then since $|W| \ge m$, by Lemma~\ref{lem:fundamental-L2}, for any fixed $x
    \in \bbR^n$, we have with probability $1-1/(16m^3)$ over the choice of $W$ that
    \begin{equation}\label{eq:btW-x}
    \|T_W\phi_x - \mcJ\phi_x\|_2 < \delta\,.
    \end{equation}
    and
    \begin{equation}\label{eq:btW-x2}
        \Prob_{z \sim D}\left(|T_W(\phi_x)(z) - \mcJ(\phi_x)(z)|
        > \delta/2\right) <
        1/(16m^3)\,.
    \end{equation}
    Therefore, by Markov's inequality, with probability $1 - 1/(2m)$ over the choice of $W$,
    Eqs.~\eqref{eq:btW-x} and~\eqref{eq:btW-x2} both hold for a random $x \sim
    D$ with probability $1-1/(8m^2)$. 

    Similar to Eq.~\eqref{eq:btW-x}, with $X$ in place of $W$ and $g$ in place
    of $\phi_x$, statement (1) of the present
    lemma holds with probability $1 - 1/(16m^3) > 1-1/(8m)$ over the choice of $X$. Furthermore,
    for any fixed $X$, taking a union bound over $W$, we have with probability
    $1-m/(16m^3) > 1 - 1/(8m)$ that statement (2) holds.

    Now suppose $W$ is such that Eq.~\eqref{eq:btW-x} holds for a random
    $x \sim D$ with
    probability at least $1-1/(8m^2)$; as we have already observed, this is the
    case with probability at least $1-1/(2m)$ over the choice of $W$. Then by a
    union bound over $X$, it then follows that with probability
    $1-1/(8m)$ over the choice of $X$, statement (3) holds. Finally,
    suppose similarly that $W$ is such that Eq.~\eqref{eq:btW-x2} holds for
    a random $x \sim D$ with probability at least $1-1/(8m^2)$. By a union bound, we with
    probability at least $1 - 1/(8m)$ that for all $x \in X$,
    \[
        \Prob_{z \sim D}\left(|T_W(\phi_x)(z) - \mcJ(\phi_x)(z)|
        > \delta/2\right) < 1/(16m^3)\,.
    \]
    Now, fixing such an $x \in X$, a union bound over all $y \in X$ with $y\ne x$ gives that
    \[
        |T_W(\phi_x)(y) - \mcJ(\phi_x)(y)| \le \delta/2
    \]
    with probability $1 - 1/(16m^2)$. Taking another union bound over all $x \in X$, we 
    get statement (4) with probability $1 - 1/(16m)$ as well.
    Overall, statements (1)--(4) hold with probability at least $1-1/m$.
\end{proof}

For the remainder of this section, we use the notation
    $\alpha_i = \max_{u \in W}|a_i(u)|$
and
    $\beta_i = \max_{x \in X}|H_i(x)|$.

We focus on the second step of our analysis, as outlined in
Section~\ref{sec:approach}, bounding
    the rate at which error from the approximations of $\mcJ$ described
    above accumulates over multiple iterations of GD.  More precisely, we control
    the norm of $f$, measured via $\alpha_i$ and
$\beta_i$. The statements are given in the following two lemmas.

\begin{lemma}\label{lem:T2-approx}
    Suppose statements (1)--(3) of Lemma~\ref{lem:TXW-props} all hold. Then for all
    $i \in \bbN$, we have both
        $\|T_XH_i - \mcJ H_i\|_2 \le \delta(m\alpha_i + 1)$
    and
        $\|\Delta H_i - \mcJ^2H_i\|_2 \le
        \delta\left(\beta_i + m\alpha_i + 1\right)$.
\end{lemma}
\begin{proof}
    By Lemma~\ref{lem:TXW-props}~(3), since $T_W$ and $\mcJ$ are linear operators,
    \begin{align*}
        \left\|(T_W - \mcJ)(T_XH_i)\right\|_2 &= \left\|(T_W -
        \mcJ)\left(\frac{1}{m}\sum_{x \in X} 
        H_i(x)\phi_x\right)\right\|_2 \\
        &\le \frac{1}{m}\sum_{x \in X} \beta_i
        \|(T_W-\mcJ)\phi_x\|_2 \le \delta \beta_i\,.
    \end{align*}
    Similarly, by Lemma~\ref{lem:TXW-props}~(2), we have
    \begin{align*}
        \left\|(T_X - \mcJ)f_i\right\|_2 &= \left\|(T_X -
        \mcJ)\left(\sum_{u \in
        W}a_i(u)\phi_u\right)\right\|_2 \\
        &\le \sum_{u \in W} \alpha_i\|(T_X-\mcJ)\phi_u\|_2
        \le m\delta \alpha_i\,.
    \end{align*}
    Finally, by Lemma~\ref{lem:TXW-props}~(1), we have
    \begin{align*}
        \|(T_X - \mcJ)H_i\|_2 &=
        \|(T_X- \mcJ)g - (T_X - \mcJ)f_i\|_2 \\
        &\le \|T_Xg - \mcJ g\|_2 + \|(T_X - \mcJ)f_i\|_2 \\
        &\le \delta + m\delta\alpha_i\,.
    \end{align*}
    By Proposition~\ref{prop:f-from-T}, we have $\Delta H_i = T_WT_X H_i$. Therefore, since
    $\|\mcJ(h)\|_2 \le \|h\|_2$ for all functions $h$, we have altogether that
    \begin{align*}
        \|\Delta H_i - \mcJ^2H_i\|_2 &= 
        \|T_WT_X H_i - \mcJ^2 H_i\|_2 \\
        &\le \|\mcJ T_XH_i - \mcJ^2H_i)\|_2 + \delta \beta_i \\
        &\le \|\mcJ(T_X - \mcJ)H_i\|_2 + \delta \beta_i \\
        &\le \|(T_X - \mcJ)H_i\|_2 + \delta \beta_i \\
        &\le \delta\left(\beta_i + m\alpha_i + 1\right)\,.\qedhere
    \end{align*}
\end{proof}

\begin{lemma}\label{lem:infty-approx}
    For all $i \ge 0$, we have $\alpha_{i+1} \le \alpha_i + \beta_i/m$.
Furthermore, if statement (4) of Lemma~\ref{lem:TXW-props} holds, then for all
$i \ge 0$, we have 
    \[
        \beta_{i+1} \le \beta_i + \|H_i\|_2 + \delta(\beta_i/2 + m\alpha_i + 1) + 2\beta_i/m\,.
    \]
\end{lemma}
\begin{proof}[Proof of Lemma~\ref{lem:infty-approx}]
    For the first inequality, we have by definition that for all $u \in W$
    \[
        a_{i+1}(u) = a_i(u) + \frac{1}{m^2}\sum_{x \in X}H_i(x)\phi(u\cdot x) \le \alpha_i + \beta_i/m\,.
    \]

For the second inequality, fix $y \in X$. Using statement~(4) of Lemma~\ref{lem:TXW-props} we compute
\begin{align*}
    |(T_W - \mcJ)H_i(y)| &= \left|(T_W - \mcJ)\left(\frac{1}{m}\sum_{x \in X}H_i(x)\phi_x\right)(y)\right| \\
        &\le \frac{1}{m}\sum_{x \in X}|H_i(x)(T_W-\mcJ)(\phi_x)(y)| \\
        &\le \frac{\beta_i}{m}\left(|(T_W - \mcJ)(\phi_y)(y)| + \sum_{y\ne x}|(T_W-\mcJ)(\phi_x)(y)|\right) \\
        &\le \frac{\beta_i}{m}\left(2 + (m-1)\frac{\delta}{2}\right)\,.
\end{align*}
By Proposition~\ref{prop:T-Linf} and the first statement of Lemma~\ref{lem:T2-approx},
\[
\|\mcJ T_XH_i\|_{\infty} \le \|T_XH_i\|_2 \le \|\mcJ H_i\|_2 + \delta(m\alpha_i + 1) \le \|H_i\|_2
    + \delta(m\alpha_i + 1)\,.
\]
Therefore, by Proposition~\ref{prop:f-from-T},
\begin{align*}
        |H_{i+1}(y)| &\le |H_i(y)| + |T_WT_XH_i(y)| \\
        &\le \beta_i + |\mcJ T_XH_i(y)| + |(T_W - \mcJ)H_i(y)| \\
        &\le \beta_i + \|H_i\|_2 + \delta(m\alpha_i + 1) + \frac{\beta_i}{m}\left(2 +
        (m-1)\frac{\delta}{2}\right) \\
        &\le \beta_i + \|H_i\|_2 + \delta(\beta_i/2 + m\alpha_i + 1) + 2\beta_i/m\,.\qedhere
\end{align*}
\end{proof}

Finally, we complete the proof.

\begin{proof}[Proof of Lemma~\ref{lem:main-technical}]
    Let $\overline{S} = \bbN\setminus S$.  We argue by induction that for all $i \le t$, as long as
    $\|H_i^{(S)}\| \ge \eps$, the following are all true:
    \begin{enumerate}
        \item[(1)] $\beta_i \le O((i+1)\|g\|_2)$
        \item[(2)] $\alpha_i \le O((i+1)^2\|g\|_2/m)$
        \item[(3)] $\|\Delta H_i - \mcJ^2(H_i)\|_2 \le O(\delta\|g\|_2(i+1)^2)$
        \item[(4)] $\|H_i\|_2 \le \|g\|_2$
    \end{enumerate}
    (So the statement of the lemma follows from estimate~(3) and $i \le t$.)

    Since $f_0 = 0$, $a_0 = 0$, and $H_0 = g$, the base cases are all trivial. Fix $0 < i
    \le t$ and assume estimates~(1)--(4) hold for all $j < i$.
    We first prove that estimate~(1) holds for $i$. Indeed, using the second statement of
    Lemma~\ref{lem:infty-approx}, and then simplifying using the inductive
    hypothesis for estimates~(1),~(2) and~(4), we have
    \begin{align*}
        \beta_i &\le \beta_{i-1} + 
        \|H_{i-1}\|_2 + \delta(\beta_{i-1}/2 + m\alpha_{i-1} + 1) + 2\beta_{i-1}/m \\
        &\le
        \beta_{i-1} + \|g\|_2 + O(\delta (i\|g\|_2 + (i+1)^2\|g\| + 1) + (i+1)\|g\|_2/m)\,.
    \end{align*}
This latter expression is at most $\beta_{i-1} + \|g\|_2 + O(\|g\|_2/t)$, using the
fact that $i < t$ and the definitions of $t$, $\delta$, and $m$ in
    Eqs.~\eqref{eq:c_t},~\eqref{eq:c_delta}, and~\eqref{eq:c_m}.
    Estimate~(1) now follows by induction.

    Similarly, from the first statement of Lemma~\ref{lem:infty-approx} and
    from estimate~(1), we have
    \[
        \alpha_i \le \alpha_{i-1} + \beta_i/m = \alpha_{i-1} + O((i+1)\|g\|_2/m),
    \]
which gives estimate~(2) by induction.

    By the second statement of Lemma~\ref{lem:T2-approx}, and using
    estimates~(1) and~(2), we have
    \[
        \|\Delta H_i - \mcJ^2H_i\|_2 \le
        \delta\left(\beta_i + m\alpha_i + 1\right) \le 
        O(\delta\|g\|_2(i+1)^2)\,,
    \]
    giving estimate~(3) by induction.
    Rewriting $\Delta H_i = H_{i+1} - H_i$, we have
    $\|H_{i+1}\|_2 \le \|H_i\| + O(\delta\|g\|_2(i+1)^2)$. Now by estimate~(4), $\|H_i - \mcJ^2
    H_i\|_2 \le \|g\|_2$, and so
    \[
        \|H_{i+1}\|_2^2 \le \|H_i - \mcJ^2 H_i\|_2^2 + O(\delta\|g\|_2^2(i+1)^2)\,.
    \]
    Combining with Lemma~\ref{lem:S-alpha-estimate},
    \[
        \|H_{i+1}\|_2^2 \le \|H_i\|_2^2 - \alpha^4\|H_i^{(S)}\|_2^2 + O(\delta\|g\|_2^2(i+1)^2)\,.
    \]
    For a sufficiently small choice of the constant $c_{\delta}$ defining $\delta$
    (Eq.~\eqref{eq:c_delta}), under the assumption that $\|H_i^{(S)}\|_2^2 \ge \eps$, we can take
    the $O(\delta\|g\|_2^2(i+1)^2)$ term to be at most $(\alpha^4/2)\|H_i^{(S)}\|_2^2$. Therefore,
    \[
        \|H_{i+1}\|_2^2 \le \|H_i\|_2^2 - \frac{\alpha^4}{2}\|H_i^{(S)}\|_2^2\,.
    \]
    The norm $\|H_i\|_2$ of the residual is therefore monotonically decreasing in $i$, giving estimate~(4).
\end{proof}

%% file: sq.tex
\section{Statistical query models}\label{sec:sq}

We now prove our statistical query lower boudns stated in 
Theorem~\ref{thm:sq-all}. First, we
remark on some of the difficulty of choosing an appropriate statistical
query model for regression problems. 
Let $D$ be a probability distribution over a domain $X$, and let
$f: X \to \bbR$ be an unknown concept. A natural and very general statistical
query model for the regression problem of learning $f$ might allow as queries
arbitrary measurable functions $h: X \times \bbR \to [0,1]$.  The SQ oracle
should respond to such a query with a value $v$ approximating $\Exp_{x \sim D}
h(x, f(x))$ to within the error tolerance. But such a model is in fact far too
general, as the following proposition shows.

\begin{proposition}\label{prop:measurable-queries}
    Let $D$ be a probability distribution over some domain $X$,
    and let $\mcC$ be a finite family of functions $f: X \to \bbR$ such that
    for every pair $f,g \in \mcC$, the probability over $x \sim D$ that $f(x) =
    g(x)$ is $0$. There is a measurable function $h$ such that for every $v \in
    [0,1]$,
    \[
        \left|\{ f \in \mcC \mid |\Exp_{x \sim D} h(x, f(x)) - v| < 
        \tau\}\right|
        = O(\tau |\mcC|)\,.
    \]
    In particular, with a constant error tolerance $\tau$, such a family $\mcC$
    can be learned using $\log |\mcC|$ statistical queries.
\end{proposition}
\begin{proof}
    We arbitrarily choose evenly spaced values $v_f \in [0,1]$ corresponding to
    each $f \in \mcC$. It suffices to find a measurable function $h(x, y)$ such
    that $h(x,f(x)) = v_f$ for all $x \in X$, excluding perhaps a subset of $X$
    of measure $0$ where two different functions have equal values. Since
    this condition specifies the value of $h$ only on a set of measure $0$ in
    $X \times \bbR$, it is straightforward to find such a function $h$.
\end{proof}

In particular, a statistical query model allowing arbitary measurable and bounded queries would
allow efficiently learning any finite class of real-valued functions, perhaps perturbed slightly
to ensure the functions disagree pairwise almost everywhere. 

Furthermore, arbitrary measurable query functions don't have concise
descriptions anyway. So it is reasonable to require ``well-behaved'' query
functions. We now describe three ``well-behaved'' statistical query settings, and prove strong
lower bounds against algorithms learning degree-$k$ polynomials on $S^{n-1}$ in each setting.

The approach taken in~\cite{SVWX17}, are revisited here, is to require the query function $h : X
\times \bbR \to [0,1]$ to be Lipschitz for every fixed $x \in X$. However, the Lipschitz-ness of
the query function is sensitive to the scale of the concepts $f: X \to \bbR$; a Lipschitz constant
of $1$ for the query function is far more meaningful when the concepts have bounded range than when
their outputs are spread across the entirety of $\bbR$. Hence, we consider
\emph{$L^{\infty}$-normalized Lipschitz queries} in Section~\ref{sec:lipschitz-lb}.

Another approach is to insist on noisy concepts. In the model considered above, the query function $h$
will receive inputs of the form $(x, f(x))$ where $x$ is drawn from a distribution over the input
space $X$ and $f: X \to \bbR$ is the unknown concept. In a noisy setting, it would be reasonable to
replace the query input with a noised version $(x, f(x) + \zeta)$ where $\zeta$ is drawn from a
noise distribution, such as a standard Gaussian. Clearly, the statistical oracle becomes weaker as
the variance of $\zeta$ increases, but as in the Lipschitz queries considered above, the
relationship between the strength of the oracle and the amount of noise is sensitive to the scale
of the concept $f: X \to \bbR$. Again, in this setting, we will consider
\emph{$L^{\infty}$-normalized queries in the presence of Gaussian noise of variance $1/\lambda^2$}.
In fact, since the noise on the input has the effect of smoothing the expectation of the query
function, this setting amounts to a special case of the more general $L^{\infty}$-normalized
Lipschitz queries described previously. However, the noise approach allows for binary query
functions $h: X \to \{0, 1\}$, so this is the model we consider for lower bounds on $\oneSTAT$,
given in Section~\ref{sec:sq-noise}.

A final plausible query model would restrict the \emph{form} of the queries.
As noted in the proof of Proposition~\ref{prop:measurable-queries}, the only
important values of a query function $h : X \times \bbR \to [0,1]$ are on the
zero measure subset $\{ (x, f(x)) : f \in \mcC, x \in X\} \subseteq X \times
\bbR$. Hence, instead of working in the product space $X \times \bbR$ where
only a measure-zero set is relevant, we might instead allow query functions $h
: X \to [0,1]$ to an oracle that responds with an approximation of the inner
product with the concept, i.e., with a value $v$ approximating $\Exp_{x \in
D}(h(x)f(x))$ where $f$ is the unknown concept. Such a query is called an \emph{inner product
query}, and these queries
suffice, for example, to train a neural network using gradient descent against mean squared
error.  Lower bounds against such queries are given in
Section~\ref{sec:ip-lb}.

\subsection{Inner product queries}\label{sec:ip-lb}


We recall the definition of statistical dimension, denoting by $\rho_{D}(\mcC)$
the average correlation among the functions of $\mcC$, i.e.,
\[
    \rho_D(\mcC) \frac{1}{|\mcC|^2}\sum_{f,g \in \mcC} \rho_D(f,g)
\]
where $\rho_D(f,g) = \Cov_D(f,g)/\sqrt{\Var(f)\Var(g)}$.

\begin{definition}\label{def:sda}
    Let $\bar{\gamma} > 0$, let $D$ be a probability distribution over some
    domain $X$, and let $\mcC$ be a family of functions $f:X \to \bbR$.
    The \textbf{statistical dimension} of $\mcC$ relative to $D$ with average
    correlation $\bar{\gamma}$, denoted by $\SDA(\mcC, D, \bar{\gamma})$, is
    defined to be the largest integer $d$ such that for every subset $\mcC'
    \subseteq \mcC$ of size at least $|\mcC'| \ge |\mcC|/d$, we have
    $\rho_D(\mcC') \le \bar{\gamma}$.
\end{definition}

The following theorem can be proved in a manner almost identical to the proof
of~\cite[Theorem 2.7]{FGRVX13}.

\begin{theorem}\label{thm:sda}
    Let $D$ be a distribution on a domain $X$ and let $\mcC$ be a family of
    functions $f: X \to \bbR$.
    Suppose there for some $d, \bar{\gamma} > 0$ we have
    $\SDA(\mcC, D, \bar{\gamma}) \ge d$.
    Let $A$ be a randomized algorithm learning $\mcC$ over $D$ with probability
    greater than $1/2$ to within regression error less than $\Omega(1)$. If $A$ only
    uses inner product queries to $\VSTAT(1/(3\bar{\gamma}))$, then $A$
    uses $\Omega(d)$ queries.
\end{theorem}

In what follows, $f^{(k)}_u : S^n \to \bbR$ is defined by $f^{(k)}_u(x) = \sqrt{N(n,k)}P_{n,k}(u
\cdot x)$.

\begin{lemma}\label{lem:correlation-random-legendre}
    Let $u,v \in S^n$ be such that $|u\cdot v| = t$. Then
    \[
        \rho(f_u^{(k)}, f_v^{(k)}) \le \left(\left(1 +
        \frac{k-1}{k+n-3}\right)|t| + \sqrt{\frac{k-1}{k+n-3}}\right)^k
    \]
\end{lemma}
\begin{proof}
    By the Funk--Hecke theorem, $\|f^{(k)}_u\|_2 = 1$ 
    We therefore have, again by the Funk--Hecke theorem,
    \begin{align*}
        \rho_{x \in S^n}(f_u^{(k)}(x), f_v^{(k)}(x)) &=
        N(n,k)\left(\Exp_{x \in S^n}(P_{n,k}(u\cdot x)P_{n,k}(v \cdot x))
        - \Exp_{x \in S^n}(P_{n,k}(u\cdot x))\Exp_{x \in S^n}(P_{n,k}(v \cdot
        x))\right) \\
        = P_{n,k}(u \cdot v) = P_{n,k}(t)\,.
    \end{align*}

    Furthermore, for all $\ell$, the Legendre polynomials satisfy the
    recurrence relation (see \cite[Proposition 3.3.11]{groemer-sh})
    \[
        (\ell + n -2)P_{n,\ell+1}(t) - (2\ell + n -2)tP_{n,\ell}(t) +
        \ell P_{n,\ell-1}(t) = 0\,.
    \]
    Since $P_{n,0}(t) = 1$ and $P_{n,1}(t) = t$ (by \cite[Proposition
    3.3.7]{groemer-sh}), the result follows.
\end{proof}

We can now prove the SQ lower bound for this class of queries.

\begin{proof}[Proof of Theorem~\ref{thm:sq-all}~(1).]
    Taking a random (uniform) set $B$ of $d$ vectors $u \in S^n$, let
    $\mcC = \{f_u^{(k)} : u \in B\}$. For any pair of distinct vectors $u,v \in B$, we have with
    probability $1/d^2$ that $t_{uv} = |u\cdot v| = O(\sqrt{(\log d)/n})$. Thus, with positive
    probability, we have $t_{uv} = O(\sqrt{(\log d)/n})$ for all distinct $u,v \in B$. Then
    by Lemma~\ref{lem:correlation-random-legendre}, we have 
    \[
        \rho(f_u^{(k)}, f_v^{(k)}) = O\left(\sqrt{\frac{\log d}{n}}\right)^k \le n^{-\Omega(k)}
    \]
    for all $u, v \in B$. The theorem now follows from Theorem~\ref{thm:sda}.
\end{proof}

\subsection{Lipschitz queries}\label{sec:lipschitz-lb}

We now recall the Lipschitz query model introduced in~\cite{SVWX17}.  The
functions learned in that paper were already bounded, so no $L^{\infty}$
normalization is performed. We state an $L^{\infty}$-normalized version of the
relationship between statistical dimension and statistical query complexity,
which are an immediate consequence of those proved in~\cite{SVWX17}.

For $y \in \bbR$ and $\eps > 0$, we define the \emph{$\eps$-soft indicator function}
$\chi_y^{(\eps)}: \bbR \to \bbR$ as
\[
\chi_y^{(\eps)}(x) = \chi_y(x) = \max\{0,1/\eps - (1/\eps)^2|x-y|\}. 
\]
So $\chi_y$ is $(1/\eps)^2$-Lipschitz, is
supported on $(y-\eps, y+\eps)$, and has norm $\|\chi_y\|_1 = 1$.

\begin{definition}\label{def:eps-sda}
    Let $\bar{\gamma} > 0$, let $D$ be a probability distribution over some
    domain $X$, and let $\mcC$ be a family of functions $f:X \to \bbR$ that
    are identically distributed as random variables over $D$.
    The \textbf{statistical dimension} of $\mcC$ relative to $D$ with average
    covariance $\bar{\gamma}$ and precision $\eps$, denoted by $\eSDA(\mcC, D,
    \bar{\gamma})$, is defined to be the largest integer $d$ such that
    the following holds:
    for every $y \in \bbR$ and every subset $\mcC' \subseteq \mcC$ of size
    $|\mcC'| > |\mcC|/d$, we have $\rho_D(\mcC') \le \bar{\gamma}$.
    Moreover, $\Cov_D(\mcC_y') \le (\max\{\eps, \mu(y)\})^2\bar{\gamma}$ where $\mcC_y' = \{\chi_y^{(\eps)} \circ f : f \in \mcC\}$ and
    $\mu(y) = \|f\|_{\infty}\Exp_D(\chi_y^{(\eps)} \circ f)$ for some $f \in \mcC$. 
\end{definition}

\begin{theorem}\label{thm:esda}
    Let $D$ be a distribution on a domain $X$ and let $\mcC$ be a family of
    functions $f: X \to \bbR$ identically distributed as unit-variance random variables
    over $D$. Suppose there is $d \in \bbR$ and $\lambda, \bar{\gamma} > 0$
    such that $\lambda \ge 1 \ge \bar{\gamma}$
    and $\eSDA(\mcC, D, \bar{\gamma}) \ge d$, where $\eps \le
    \bar{\gamma}/(2\lambda)$.
    Let $A$ be a randomized algorithm learning $\mcC$ over $D$ with probability
    greater than $1/2$ to regression error less than $\Omega(1) -
    2\sqrt{\bar{\gamma}}$. If $A$ only uses $L^{\infty}$-normalized queries to
    $\VSTAT(t)$ for some $t = O(1/\bar{\gamma})$, which are $\lambda$-Lipschitz
    at any fixed $x \in X$, then $A$ uses $\Omega(d)$ queries.
\end{theorem}

We can now prove the lower bound for this query model. We use the same family
of functions as for the inner product query model, but we must now also
estimate the covariances of the soft indicators of these functions, as in the
following lemma.  We recall our notation $f^{(k)}_u(x) = \sqrt{N(n,k)}P_{n,k}(u \cdot x)$.

\begin{lemma}\label{lem:covariance}
    Let $u, v \in S^{n}$, let $y \in \bbR$, 
    and let $\ell, \eps > 0$. Let
    \[
        \mu_0 = \Exp_{x \sim S^n}(\chi_y^{(\eps)}(f^{(k)}_u(x)))\,.
    \]
    Then
\[
    \left|\Cov_{x \sim
    S^n}(\chi_y^{(\eps)}(f^{(k)}_u(x)), \chi_y^{(\eps)}(f^{(k)}_v(x)))\right| =
    O(\ell(u\cdot v)^2\log n\mu_0^2 + n^{-\ell}/\eps^2)
\]
\end{lemma}
\begin{proof}
    For $w,x \in S^n$, we write $z_w(x) = \chi_y^{(\eps)}(f^{(k)}_w(x))$.

    We will pass from $x \sim S^n$ to $x$ sampled from the subset $R$ of the
    sphere such that $x \cdot u$ and $x \cdot v$ are each in the range 
    $[-a,a]$ where $a = \sqrt{\frac{2\ell\log n}{n}}$. Let $E = \Exp_{x \sim
    R}(z_u(x)z_v(x))$.  We wish to estimate the quantity
    \begin{align*}
        \Exp_{x \sim S^n}(z_u(x)z_v(x)) &= E +
        \Exp_{x \sim S^n\setminus R}(z_u(x)z_w(x)) \\
        &\le E + (1/\eps)^2\vol(S^n\setminus R)
        \le E + O(n^{-\ell}/\eps^2) \,.
    \end{align*}

    Next, let $\alpha = u\cdot v$ and write $v = \sqrt{1-\alpha^2}v'+
    \alpha u$ where $v' \perp u$ and $u,v'$ can be completed to some basis
    $w_1=u,w_2=v',w_3,\ldots, w_n$ of $\R^n$. For a point $x$, let the coordinates 
    in this basis be $x_1, x_2, \ldots, x_n$. In the set $R$, we have $x_1 \in
    [-a,a]$ and similarly $\alpha x_1 + \sqrt{1-\alpha^2}x_2 \in [-a,a]$.
    
    Let $\zeta(t) =
    \chi_y^{(\eps)}(\sqrt{N(n,k)}P_{n,k}(t))$, so $z_w(x) = \zeta(w \cdot x)$.
    Then
    \begin{align*}
        E &= \Exp_{x \sim R}(z_u(x)z_v(x)) \\
        &= \frac{\vol(S^{n-2})}{\vol(S^n)}\int_{x_1=-a}^a\int_{x_2=(-a-\alpha
        x_1)/\sqrt{1-\alpha^2}}^{(a-\alpha x_1)/\sqrt{1-\alpha^2}}
        \zeta(x_1)\zeta(\sqrt{1-\alpha^2}x_1 + \alpha x_2)
        (1-x_1^2-x_2^2)^{(n-3)/2} \, dx_1dx_2\,.
    \end{align*}
    Now substituting $w = \sqrt{1-\alpha^2}x_2 + \alpha x_1$, we have
    \begin{align*}
        &\int_{x_1=-a}^a\int_{x_2=(-a-\alpha x_1)/\sqrt{1-\alpha^2}}^{(a-\alpha
        x_1)/\sqrt{1-\alpha^2}}
        \zeta(x_1)\zeta(\sqrt{1-\alpha^2}x_1 + \alpha x_2)
        (1-x_1^2-x_2^2)^{(n-3)/2} \, dx_1dx_2\ \\
        &=
        \frac{1}{\sqrt{1-\alpha^2}}\int_{x_1=-a}^a\int_{w=-a}^a \zeta(x_1)\zeta(w)
        \left(1-\frac{1+\alpha^2}{1-\alpha^2}x_1^2-\frac{w^2}{1-\alpha^2}+\frac{2\alpha
        x_1w}{1-\alpha^2}\right)^{(n-3)/2} \, dx_1dw\\
        &\le  \frac{1}{\sqrt{1-\alpha^2}}\int_{x_1=-a}^{a}
        \int_{w=-a}^{a}\zeta(x_1)\zeta(w)\left(1-x_1^2(1-\alpha^2)-w^2(1-2\alpha^2)\right)^{(n-3)/2} \, dx_1dw
\end{align*}
Next we note that in our range of $x_1, w$,
\[
\max \frac{\left(1-x_1^2(1-\alpha^2)-w^2(1-2\alpha^2)\right)^{(n-3)/2}}{\left(1-x_1^2-w^2\right)^{(n-3)/2}} 
 \le \left(1+\frac{\alpha^2 (x_1^2+2w^2)}{1-x_1^2-w^2}\right)^{(n-3)/2}
\le 1+O(\ell\alpha^2\log n)
\]
using the fact that $x_1, w \in [-\sqrt{2\ell\log n/n},\sqrt{2\ell\log n/n}]$.
Therefore,
\begin{align*}
        &\frac{1}{\sqrt{1-\alpha^2}}\int_{x_1=-a}^{a}
        \int_{w=-a}^{a}\zeta(x_1)\zeta(w)\left(1-x_1^2(1-\alpha)-w^2(1-2\alpha)\right)^{(n-3)/2}
        \, dx_1dw \\
    &\le (1+O(\ell\alpha^2\log n)) \int_{x_1=-a}^{a}
        \int_{w=-a}^{a}\zeta(x_1)\zeta(w)\left(1-x_1^2-w^2\right)^{(n-3)/2} \, dx_1dw \\
    &\le (1+O(\ell\alpha^2\log n))\int_{x_1=-a}^a \int_{w=-a}^{a}
    \zeta(x_1)\zeta(w)(1-x_1^2)^{(n-3)/2}(1-w^2)^{(n-3)/2} \, dx_1dw. 
\end{align*}
Hence, 
\[
    \E_{x \sim R}(z_u(x)z_v(x)) \le (1+O(\ell(u\cdot v)^2\log n))\mu^2.
\]
We conclude that
\[
    \left|\Cov_{x \sim
    S^n}(z_u(x), z_v(x))\right| =
    O(\ell(u\cdot v)^2\log n)\mu_0^2 + O(n^{-\ell}/\eps^2)
\]
as desired.
\end{proof}

\begin{proof}[Proof of Theorem~\ref{thm:sq-all}~(2).]
    Taking a random (uniform) set $B$ of $d$ vectors $u \in S^n$, let
    $\mcC = \{f_u^{(k)} : u \in B\}$. As seen in the proof of
    Theorem~\ref{thm:sq-all}~(1), we can take every pair $u,v \in B$ to satisfy
    $u\cdot v = O(\sqrt{(\log d)/ n})$. By
    Lemma~\ref{lem:correlation-random-legendre}, we have
    \[
        \rho(f_u^{(k)}, f_v^{(k)}) = O\left(\sqrt{\frac{\log d}{n}}\right)^k \le
        n^{-\Omega(k)}
    \]
    Fix $y \in \bbR$ and let
    \[
        \mu_0 = \Exp_{x \sim S^n}(\chi_y^{(\eps)}(f^{(k)}_u(x)))\,.
    \]
    and
    \[
        \mu(y) = \|f_u^{(k)}\|_{\infty}\mu_0  = n^{\Theta(k)}\mu_0
    \]
    Take $\bar{\gamma} = n^{-\Theta(k)}$ and 
    $\eps = \bar{\gamma}/(2\lambda)$. By
    Lemma~\ref{lem:covariance},
    using $\ell = \Theta(k\log \lambda)$ in the statement of the lemma, we furthermore have
    \[
    \left|\Cov_{x \sim
    S^n}(z_u(x), z_v(x))\right| =
    O(\ell(u\cdot v)^2\log n)\mu_0^2 + O(n^{-\ell}/\eps^2)
    = n^{-\Omega(k)}(\mu_0^2 + \eps^2)\,.
\]
    We therefore have $\SDA(\mcC, S^n, n^{-\Omega(k)}) = d$. The result now follows by
    Theorem~\ref{thm:esda}.
\end{proof}

\subsection{Queries with Gaussian noise}\label{sec:sq-noise}

We conclude this section with our lower bounds against $\oneSTAT$. These lower bounds rely on the
simulation of $\oneSTAT$ using $\VSTAT$ proved in~\cite[Theorem 3.13]{FGRVX13}.

In the regression context we have so far considered, we have a family $\mcC$ of concepts $g: X \to
\bbR$ for some domain $X$ and distribution $D$ on $X$, and the SQ query functions $h : X \times
\bbR \to \bbR$ take pairs $(x, g(x))$ (or $(x, g(x)/\|g\|_{\infty})$ in the $L^{\infty}$-normalized
case) as input. We now consider SQ algorithms that make queries $h: X \times \bbR \to \bbR$ that
will take pairs $(x, g(x) + \zeta)$ as input, where $\zeta$ is some random noise. Equivalently, we
could consider replacing $\mcC$ with noisy concepts $g$, for which $g(x)$ is not a precise number
but rather a distribution. Formally, we say an SQ algorithm makes \emph{$L^{\infty}$-normalized
queries to $\oneSTAT$ in the presence of Gaussian noise of variance $\eps$} if it makes queries of
the form $h: X \times \bbR \to \{0,1\}$, to which the oracle replies with the value of
$h(x, g(x)/\|g\|_{\infty}+\zeta)$ where $x \sim D$ and $\zeta \sim \mcN(0,\eps)$. 
We similarly define queries to $\VSTAT$ in the presence of Gaussian noise, replacing the codomain
of $h$ with $[0,1]$; in this case, as previously, the oracle must reply with a value $v$ such that
\[
    \left|p - v\right| \le \max \left\{\frac{1}{t},
        \sqrt{\frac{p(1-p)}{t}}\right\}.
\]
where 
\[
    p = \Exp_{x \sim D, \zeta \sim \mcN(0,\eps)}h(x, g(x)/\|g\|_{\infty}+\zeta)\,.
\]

In order to give our lower bounds against $\oneSTAT$ oracles, we first give the following lower
bounds for queries to $\VSTAT$ in the presence of Gaussian noise, which by
Lemma~\ref{lem:noise-lip} are in effect a special case of those proved in
Section~\ref{sec:lipschitz-lb} for Lipschitz queries.

\begin{lemma}\label{lem:vstat-gaussian}
    Let $\eps > 0$. For all $k, \lambda > 0$ and all
    sufficiently large $n$ and $d < \exp(n^{1/2-\eps})$, there exists
    a family $\mcC$ of degree-$k$ polynomials on $S^n$ with $|\mcC| = d$ such that
    if a randomized SQ algorithm learns $\mcC$ to regression error less than any fixed constant with
    probability at least $1/2$,
        it requires at least $\Omega(d)$ queries, if the queries are
    $L^{\infty}$-normalized queries to
    $\VSTAT(n^{\Omega(k)}/\lambda)$ in the presence of Gaussian noise of variance $1/\lambda^2$.
    (All the hidden constants depend on $\eps$ only.)
\end{lemma}

\begin{lemma}\label{lem:noise-lip}
    Let $D$ be a probability distribution over some domain $X$,
    let $f: X \to \bbR$, and let $\zeta \sim \mcN(0, \sigma^2)$ be a Gaussian
    random variable. Let $h: X \times \bbR \to [0,1]$ and define
    \[
        \tilde{h}(x, y) = \Exp_{\zeta} h(x, y+\zeta)\,.
    \]
    Then $\tilde{h}(x,y)$ is $1/(2\sigma)$-Lipschitz in $y$.
\end{lemma}
\begin{proof}
    Fix $x \in X$ and let $y_1, y_2 \in \bbR$. We estimate
    \begin{align*}
        |\tilde{h}(x,y_1) - \tilde{h}(x, y_2)| &= \|\Exp_{z_1 \sim \mcN(y_1, \sigma^2)}h(x,
        z_1) - \Exp_{z_2 \sim \mcN(y_2, \sigma^2)}h(x, z_2)\| \\
        &\le D_{\textrm{TV}}(\mcN(y_1, \sigma^2), \mcN(y_2, \sigma^2))
    \end{align*}
    where $D_{\textrm{TV}}$ denotes the total variation distance. This distance is bounded above by
    $|y_1 - y_2|/(2\sigma)$ (see, e.g.,~\cite{devroye2018total}).
\end{proof}

\begin{proof}[Proof of Lemma~\ref{lem:vstat-gaussian}]
    Given an $L^{\infty}$-normalized query $h : x \times \bbR \to [0,1]$ to $\VSTAT$ in the
    presence of Gaussian noise of variance $1/\lambda^2$, let $\tilde{h}$ be as in
    Lemma~\ref{lem:noise-lip}. We have
    \[
    \Exp_{x \sim D, \zeta \sim \mcN(0,\eps)}h(x, g(x)/\|g\|_{\infty}+\zeta)
    = \Exp_{x \sim D}\tilde{h}(x, g(x))\,.
\]
    Therefore, since $\tilde{h}$ is $1/(2\lambda)$-Lipschitz by Lemma~\ref{lem:noise-lip}, the
    query $h$ to $\VSTAT$ in the presence of Gaussian noise can be simulated by a
    $1/(2\lambda)$-Lipschitz query to $\VSTAT$. The result is now immediate from
    Theorem~\ref{thm:sq-all}~(2).
\end{proof}

\begin{proof}[Proof of Theorem~\ref{thm:sq-all}~(3)]
    By \cite[Theorem 3.13]{FGRVX13}, if there is an algorithm solving the problem using $m$ queries
    to $\oneSTAT$, there is an algorithm solving the problem using $m$ queries to $\VSTAT(O(m))$.
    By Lemma~\ref{lem:vstat-gaussian}, at least $\Omega(d)$ queries to
    $\VSTAT(n^{\Omega(k)}/\lambda)$ are required. Hence, at least $n^{\Omega(k)}/\lambda$ queries
    to $\oneSTAT$ are also required in the presence of Gaussian noise of variance $1/\lambda^2$, as
    long as the number $d$ of polynomials in $\mcC$ is at least $n^{\Omega(k)}$.
\end{proof}

%% file: with-sq.bbl
\begin{thebibliography}{10}

\bibitem{apvz14}
Alexandr Andoni, Rina Panigrahy, Gregory Valiant, and Li~Zhang.
\newblock Learning polynomials with neural networks.
\newblock In {\em International Conference on Machine Learning}, pages
  1908--1916, 2014.

\bibitem{arora14}
Sanjeev Arora, Aditya Bhaskara, Rong Ge, and Tengyu Ma.
\newblock Provable bounds for learning some deep representations.
\newblock In Eric~P. Xing and Tony Jebara, editors, {\em Proceedings of the
  31st International Conference on Machine Learning}, volume~32 of {\em
  Proceedings of Machine Learning Research}, pages 584--592, Bejing, China,
  22--24 Jun 2014. PMLR.

\bibitem{bach17}
Francis Bach.
\newblock Breaking the curse of dimensionality with convex neural networks.
\newblock {\em Journal of Machine Learning Research}, 18(19):1--53, 2017.

\bibitem{barron1993universal}
Andrew~R Barron.
\newblock Universal approximation bounds for superpositions of a sigmoidal
  function.
\newblock {\em IEEE Transactions on Information theory}, 39(3):930--945, 1993.

\bibitem{BR92}
Avrim Blum and Ronald~L. Rivest.
\newblock Training a 3-node neural network is {NP}-complete.
\newblock {\em Neural Networks}, 5(1):117--127, 1992.

\bibitem{bg17}
Alon Brutzkus and Amir Globerson.
\newblock Globally optimal gradient descent for a convnet with gaussian inputs.
\newblock {\em CoRR}, abs/1702.07966, 2017.

\bibitem{cybenko1989approximation}
George Cybenko.
\newblock Approximation by superpositions of a sigmoidal function.
\newblock {\em Mathematics of Control, Signals, and Systems (MCSS)},
  2(4):303--314, 1989.

\bibitem{Daniely17}
Amit Daniely.
\newblock {SGD} learns the conjugate kernel class of the network.
\newblock In {\em NIPS}, pages 2419--2427, 2017.

\bibitem{Daniely16}
Amit Daniely and Shai Shalev{-}Shwartz.
\newblock Complexity theoretic limitations on learning {DNF}'s.
\newblock In {\em Conference on Learning Theory}, pages 815--830, 2016.

\bibitem{devroye2018total}
Luc Devroye, Abbas Mehrabian, and Tommy Reddad.
\newblock The total variation distance between high-dimensional gaussians.
\newblock {\em arXiv preprint arXiv:1810.08693}, 2018.

\bibitem{eldan2016power}
Ronen Eldan and Ohad Shamir.
\newblock The power of depth for feedforward neural networks.
\newblock In {\em Conference on Learning Theory}, pages 907--940, 2016.

\bibitem{FGRVX13}
Vitaly Feldman, Elena Grigorescu, Lev Reyzin, Santosh Vempala, and Ying Xiao.
\newblock Statistical algorithms and a lower bound for planted clique.
\newblock In {\em Proceedings of the 45th annual ACM symposium on Symposium on
  theory of computing}, pages 655--664. ACM, 2013.

\bibitem{ge2018learning}
Rong Ge, Rohith Kuditipudi, Zhize Li, and Xiang Wang.
\newblock Learning two-layer neural networks with symmetric inputs.
\newblock {\em arXiv preprint arXiv:1810.06793}, 2018.

\bibitem{GLM17}
Rong Ge, Jason~D. Lee, and Tengyu Ma.
\newblock Learning one-hidden-layer neural networks with landscape design.
\newblock {\em CoRR}, abs/1711.00501, 2017.

\bibitem{GKKT16}
Surbhi Goel, Varun Kanade, Adam~R. Klivans, and Justin Thaler.
\newblock {Reliably Learning the ReLU in Polynomial Time}.
\newblock In {\em Conference on Learning Theory}, pages 1004--1042, 2017.

\bibitem{goel2018learning}
Surbhi Goel, Adam Klivans, and Raghu Meka.
\newblock Learning one convolutional layer with overlapping patches.
\newblock {\em arXiv preprint arXiv:1802.02547}, 2018.

\bibitem{Goel2017}
Surbhi Goel and Adam~R. Klivans.
\newblock Learning depth-three neural networks in polynomial time.
\newblock {\em CoRR}, abs/1709.06010, 2017.

\bibitem{groemer-sh}
Helmut Groemer.
\newblock {\em Geometric applications of Fourier series and spherical
  harmonics}, volume~61.
\newblock Cambridge University Press, 1996.

\bibitem{HardtMR16}
Moritz Hardt, Tengyu Ma, and Benjamin Recht.
\newblock Gradient descent learns linear dynamical systems.
\newblock {\em CoRR}, abs/1609.05191, 2016.

\bibitem{hornik1989multilayer}
Kurt Hornik, Maxwell Stinchcombe, and Halbert White.
\newblock Multilayer feedforward networks are universal approximators.
\newblock {\em Neural networks}, 2(5):359--366, 1989.

\bibitem{Janzamin15}
Majid Janzamin, Hanie Sedghi, and Anima Anandkumar.
\newblock Generalization bounds for neural networks through tensor
  factorization.
\newblock {\em CoRR}, abs/1506.08473, 2015.

\bibitem{Kearns93}
Michael~J. Kearns.
\newblock Efficient noise-tolerant learning from statistical queries.
\newblock In {\em Proceedings of the Twenty-Fifth Annual {ACM} Symposium on
  Theory of Computing, May 16-18, 1993, San Diego, CA, {USA}}, pages 392--401,
  1993.

\bibitem{Klivans16}
Adam~R. Klivans.
\newblock Cryptographic hardness of learning.
\newblock In {\em Encyclopedia of Algorithms}, pages 475--477. 2016.

\bibitem{LiYuan17}
Yuanzhi Li and Yang Yuan.
\newblock Convergence analysis of two-layer neural networks with relu
  activation.
\newblock In {\em Advances in Neural Information Processing Systems 30: Annual
  Conference on Neural Information Processing Systems 2017, 4-9 December 2017,
  Long Beach, CA, {USA}}, pages 597--607, 2017.

\bibitem{lss-train14}
Roi Livni, Shai Shalev{-}Shwartz, and Ohad Shamir.
\newblock On the computational efficiency of training neural networks.
\newblock pages 855--863, 2014.

\bibitem{spectral-bias}
Nasim Rahaman, Devansh Arpit, Aristide Baratin, Felix Draxler, Min Lin, Fred~A
  Hamprecht, Yoshua Bengio, and Aaron Courville.
\newblock On the spectral bias of deep neural networks.
\newblock {\em arXiv preprint arXiv:1806.08734}, 2018.

\bibitem{Rahimi-Recht08}
Ali Rahimi and Benjamin Recht.
\newblock Weighted sums of random kitchen sinks: Replacing minimization with
  randomization in learning.
\newblock In {\em NIPS}, pages 1313--1320, 2009.

\bibitem{Sedghi16}
Hanie Sedghi, Majid Janzamin, and Anima Anandkumar.
\newblock Provable tensor methods for learning mixtures of generalized linear
  models.
\newblock In {\em Proceedings of the 19th International Conference on
  Artificial Intelligence and Statistics, {AISTATS} 2016, Cadiz, Spain, May
  9-11, 2016}, pages 1223--1231, 2016.

\bibitem{Shamir16}
Ohad Shamir.
\newblock Distribution-specific hardness of learning neural networks.
\newblock {\em CoRR}, abs/1609.01037, 2016.

\bibitem{SVWX17}
Le~Song, Santosh Vempala, John Wilmes, and Bo~Xie.
\newblock On the complexity of learning neural networks.
\newblock In {\em NIPS}, pages 5520--5528, 2017.

\bibitem{telgarsky2016benefits}
Matus Telgarsky.
\newblock Benefits of depth in neural networks.
\newblock In {\em Conference on Learning Theory}, pages 1517--1539, 2016.

\bibitem{yang2001learning}
Ke~Yang.
\newblock On learning correlated boolean functions using statistical queries.
\newblock In {\em International Conference on Algorithmic Learning Theory},
  pages 59--76. Springer, 2001.

\bibitem{zhong17}
Kai Zhong, Zhao Song, Prateek Jain, Peter~L. Bartlett, and Inderjit~S. Dhillon.
\newblock Recovery guarantees for one-hidden-layer neural networks.
\newblock In Doina Precup and Yee~Whye Teh, editors, {\em Proceedings of the
  34th International Conference on Machine Learning}, volume~70 of {\em
  Proceedings of Machine Learning Research}, pages 4140--4149, International
  Convention Centre, Sydney, Australia, 06--11 Aug 2017. PMLR.

\end{thebibliography}
